\documentclass[lettersize,journal]{IEEEtran}

\usepackage{amsmath,amsfonts}
\usepackage{algorithmic}
\usepackage{algorithm}
\usepackage{array}
\usepackage{subfigure}
\usepackage{subcaption}
\usepackage{textcomp}
\usepackage{stfloats}
\usepackage{url}
\usepackage{verbatim}
\usepackage{graphicx}
\usepackage{amsthm}
\usepackage{multirow}
\usepackage{todonotes}
\usepackage{bbding}
\usepackage{pifont}
\usepackage{url}
\usepackage{xr}
\usepackage{hyperref}
\usepackage{comment}
\usepackage{color}
\usepackage{soul}

\graphicspath{{pictures/}}

\begin{document}

\title{Deep Frequency Awareness Functional Maps for Robust Shape Matching}

\author{Feifan~Luo,
        Qinsong~Li,
        Ling~Hu,
        Haibo~Wang, 
        Haojun~Xu,
        Xinru~Liu, \protect\\
        Shengjun~Liu,~\IEEEmembership{Member,~IEEE,} and
        Hongyang~Chen,~\IEEEmembership{Senior Member,~IEEE}

    \IEEEcompsocitemizethanks{
    \IEEEcompsocthanksitem Equal contribution: Feifan~Luo and Qinsong~Li.
    \IEEEcompsocthanksitem Corresponding authors: Shengjun~Liu and Hongyang~Chen.
    \IEEEcompsocthanksitem Feifan Luo is both with the School of Mathematics and Statistics, Central South University, Changsha, P.R.China and the College of Computer Science and Technology, Zhejiang University, Hangzhou, P.R.China. Email: Luoff@zju.edu.cn.
     \IEEEcompsocthanksitem Qinsong~Li is with the Big Data Institute, Central South University, Changsha, P.R.China.
    Email: qinsli.cg@foxmail.com.
    \IEEEcompsocthanksitem Shengjun~Liu, Haibo~Wang, Haojun~Xu and~Xinru~Liu are with the Institute of Engineering Modeling and Scientific Computing, and School of Mathematics and Statistics, Central South University, Changsha, P.R.China. 
    	E-mail: shjliu.cg@csu.edu.cn, wang\_haibo2017@163.com, xuhaojun4@gmail.com, liuxinru@csu.edu.cn.
     \IEEEcompsocthanksitem Hongyang~Chen is with the Research Center for Data Hub and Security, Zhejiang Lab, Hangzhou, P.R.China.
    	Email: dr.h.chen@ieee.org; hongyang@zhejianglab.com.
     \IEEEcompsocthanksitem Ling~Hu is with the School of Mathematics and Statistics, Hunan First Normal University, Changsha, P.R.China.
    	Email: huling.cg@foxmail.com.
     }
\thanks{Manuscript received June 16, 2024; revised MM DD, YY.}}

% The paper headers
\markboth{Journal of \LaTeX\ Class Files,~Vol.~14, No.~8, August~2021}%
{Shell \MakeLowercase{\textit{et al.}}: A Sample Article Using IEEEtran.cls for IEEE Journals}

\maketitle
\begin{abstract}
Traditional deep functional map frameworks are widely used for 3D shape matching; however, many methods fail to adaptively capture the relevant frequency information required for functional map estimation in complex scenarios, leading to poor performance, especially under significant deformations. To address these challenges, we propose a novel unsupervised learning-based framework, Deep Frequency Awareness Functional Maps (DFAFM), specifically designed to tackle diverse shape-matching problems. Our approach introduces the Spectral Filter Operator Preservation constraint, which ensures the preservation of critical frequency information. These constraints promote frequency awareness by learning a set of spectral filters and incorporating them as a loss function to jointly supervise the functional maps, pointwise maps, and spectral filters. The spectral filters are constructed using orthonormal Jacobi polynomials with learnable coefficients, enabling adaptive and efficient frequency representation. Furthermore, we propose a refinement strategy that leverages the learned spectral filters and constraints to enhance the accuracy of the final pointwise map. Extensive experiments conducted on multiple benchmark datasets demonstrate that our method outperforms state-of-the-art approaches, particularly in challenging scenarios involving non-isometric deformations and inconsistent topology. Our code is available at \url{https://github.com/LuoFeifan77/DeepFAFM}.
\end{abstract}


\begin{figure}[h!t]
	\centering
	\includegraphics[width=1\linewidth]{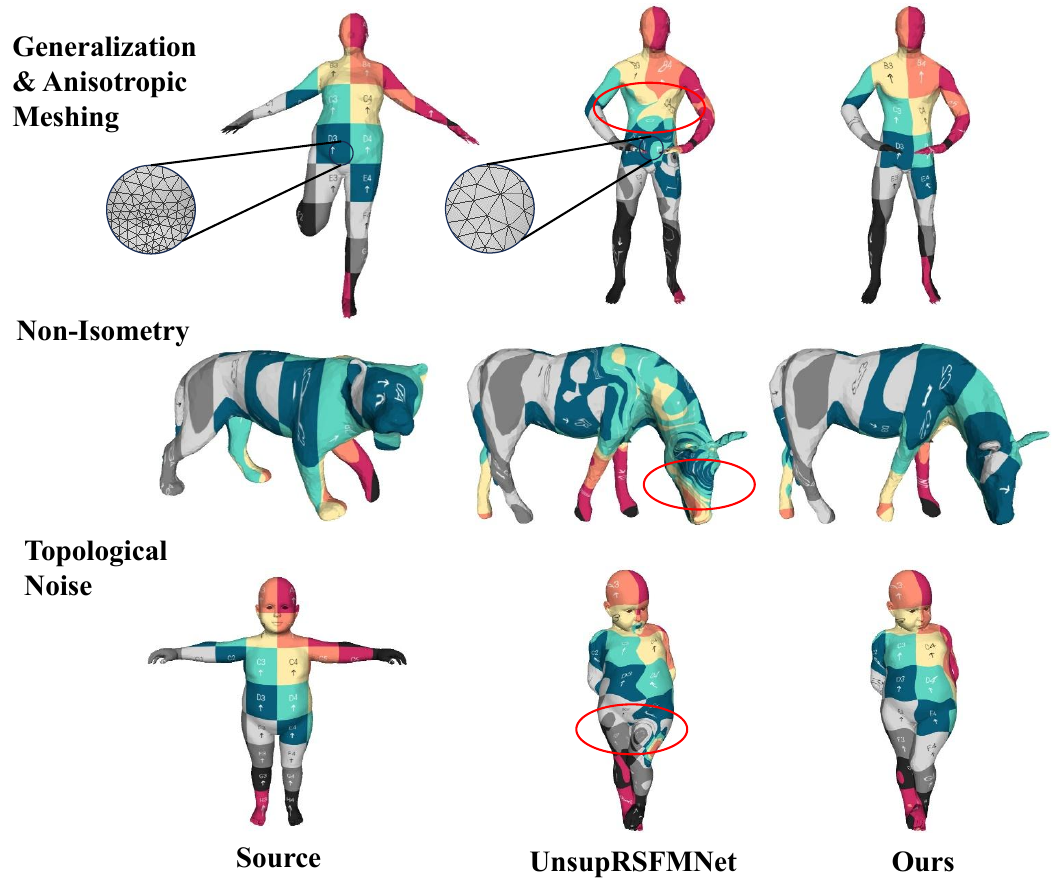}
	\caption{We propose a novel unsupervised spectral shape matching approach that is more robust than UnsupRSFMNet \cite{Cao2023} across a broad range of challenging settings: shape matching with anisotropic meshing and generalization (training on remeshed SCAPE and testing on anisotropic remeshed FAUST), shape matching of non-isometric shape pairs, shape matching with topological noise. }
	\label{fig1: comparisons with UnsupRSFMNet}
\end{figure}

\section{Introduction}\label{sec:Intro}
\IEEEPARstart{N}{on-rigid} shape matching is a fundamental problem in shape analysis and related fields, focusing on establishing meaningful pointwise correspondences or maps between shapes. This task has a wide range of applications, including deformation transfer~\cite{Sumner2004}, shape interpolation~\cite{Eisenberger2021}, and statistical shape analysis~\cite{Bogo2014FAUST}.

In recent years, numerous methods have been developed for shape correspondence. Among these, the functional maps framework~\cite{Ovsjanikov2012} has emerged as one of the most influential due to its flexibility and efficient representation. Functional maps facilitate the transfer of functions between shapes and enable efficient optimization and interconversion with pointwise maps. Subsequent works~\cite{Nogneng2017,Litany2017,Rodola2017,melzi2019matching,eisenberger2020smooth,Hu2021,Donati2022,2022SmoothNonRigidShapeMatchingviaEffectiveDirichletEnergyOptimization,2023ElasticBasis,sahilliouglu2022augmented} have further improved the framework by incorporating various preservation constraints and advanced technologies.

Unlike the functional map-based methods that utilize handcrafted features, the seminal work FMNet (also known as the deep functional maps)~\cite{litany2017deep} was the first to use the learned features from shapes as the optimal descriptors to produce desired functional maps. This approach significantly enhanced the learning process by integrating geometric priors. Building on this foundation, numerous learning-based methods have been proposed for non-rigid 3D shape matching, including supervised\cite{donati2020deep,Sharp2020}, semi-supervised~\cite{sharma2020weakly}, and unsupervised approaches~\cite{roufosse2019unsupervised,halimi2019unsupervised,cao2024revisiting}. 

Recent advancements have also extended from full shape correspondence~\cite{liu2022wtfm,Sun_2023_ICCV} to partial shape correspondence~\cite{attaiki2021dpfm,HU2023101189,Cao2023}. Despite these advances, existing deep functional map methods struggle to adaptively capture the relevant frequency information needed for functional map estimation in complex matching scenarios. These approaches treat all frequency information equally during functional map computation, often overlooking critical frequency components that positively impact matching performance. Consequently, unimportant frequency information may be retained, leading to suboptimal performance in scenarios involving large deformations.

Motivated by the above discussions, we introduce spectral filter operator preservation constraints, supported by rigorous theoretical guarantees. These constraints ensure that frequency information within spectral filter operators is preserved during functional map computation. Crucially, the spectral filter operators can be flexibly generated in a data-driven manner by adjusting a set of spectral filter functions. This adaptability allows the constraint to selectively preserve significant frequency information, leading to more accurate correspondences in specific matching scenarios. Building on this foundation, we propose a novel learning-based framework called Deep Frequency Awareness Functional Maps (DFAFM) for shape matching, which is more robust than UnsupRSFMNet\cite{Cao2023}, as shown in Fig. \ref{fig1: comparisons with UnsupRSFMNet}. In our approach, the proposed constraints are formulated as the unsupervised loss function to jointly optimize pointwise maps, functional maps, and a set of spectral filters. The spectral filters are derived from the orthonormal Jacobi polynomials, enabling the adaptive focus on critical frequency features. This strategy simultaneously optimizes functional and pointwise maps while learning a collection of enhanced spectral filters tailored to specific tasks. Finally, we incorporate the constraint and the learned spectral filters into a fast refinement technique, producing more accurate and robust results during inference.

Our main contributions are summarized as follows:
\begin{itemize}
    \item We propose a general constraint, Spectral Filter Operator Preservation, for determining functional maps, where spectral filters can be optimized in a data-driven manner to adaptively preserve important frequency information.
    \item By integrating this constraint as a loss function, we develop the Deep Frequency Awareness Functional Maps (DFAFM) for shape matching and introduce an efficient refinement strategy that utilizes learned spectral filters to improve correspondence accuracy.
    \item Extensive experiments demonstrate that our method significantly improves correspondence quality across various shape-matching scenarios, particularly in highly challenging non-isometric and inconsistent topology settings.
\end{itemize}

\section{Related work}\label{sec:RW}
We refer readers to the survey \cite{sahilliouglu2020recent} for an in-depth view of shape correspondence. Below we review the methods most related to ours.

\subsection{Axiomatic functional map methods}
Axiomatic shape-matching methods typically rely on geometric assumptions and correspondence criteria defined as optimization objective functions. Approaches such as establishing correspondences between shapes by pointwise descriptors ~\cite{sun2009concise,lhllcgf2021,Cosmo2022,Liu2023} or introducing correspondence distortion metrics using pairwise descriptors~\cite{Vestner2017,Xiang_2020_CVPR} often overlook continuity in map computation and involve optimization objective functions that are inherently nonconvex, posing a quadratic assignment problem that is NP-hard. In contrast, the functional map framework~\cite{Ovsjanikov2012} encodes the pointwise map as a low-dimensional, compact matrix that can be efficiently optimized, making it a popular choice for solving non-rigid shape-matching problems. This framework has been extensively extended in subsequent works~\cite{Huang2014,Nogneng2017,Rodola2017,Litany2017,ren2018continuous,Ren2019,melzi2019matching,Hu2021,ren2021discrete,Huang2020,Gao_2021_CVPR,fan2022coherent,Donati2022,2022SmoothNonRigidShapeMatchingviaEffectiveDirichletEnergyOptimization,2023ElasticBasis}. A common strategy in these works is to incorporate robust functional or map constraints to improve accuracy and robustness, such as Laplace-Beltrami operator (LBO) commutativity~\cite{Ovsjanikov2012}, descriptor preservation~\cite{Nogneng2017}, orientation preservation~\cite{ren2018continuous,Donati2022}, heat kernel preservation~\cite{vestner2017kernelmatching}, wavelet preservation~\cite{Hu2021}, and smoothness in pointwise mapping~\cite{2022SmoothNonRigidShapeMatchingviaEffectiveDirichletEnergyOptimization}.

Distinct from these methods, we propose a novel and powerful constraint that selectively preserves critical frequency information for functional map estimation in a data-driven manner. Furthermore, we demonstrate that heat kernel preservation~\cite{vestner2017kernelmatching}, ZoomOut~\cite{Ren2019}, and wavelet preservation~\cite{Hu2021} are special cases of our approach, using specific spectral filters within our generalized framework.

\subsection{Deep functional map methods}
In contrast to axiomatic functional map methods, which heavily rely on the quality of initial correspondences or descriptors, deep functional map methods aim to overcome this limitation by learning optimal descriptors directly from training data. The pioneering work FMNet\cite{litany2017deep} introduced this concept by using the SHOT descriptor\cite{salti2014shot} as input and optimizing it through residual multilayer perceptron (MLP) layers. UnsupFMNet\cite{halimi2019unsupervised} extended this approach by introducing an unsupervised loss function that minimizes pairwise geodesic distance distortion, achieving performance comparable to FMNet. To address the computational overhead of geodesic distance calculations, \textit{Ayguen et al.}\cite{Ayguen2020} replaced geodesic distances in UnsupFMNet with a more efficient heat kernel, maintaining similar results with reduced complexity. All the above methods constrain the pointwise map as part of their loss functions. An alternative approach is to directly constrain the functional map, as demonstrated by works such as~\cite{roufosse2019unsupervised,sharma2020weakly}, which impose properties like bijection, orthogonality, and commutativity with the LBOs. Subsequent research has primarily focused on improving feature extraction and functional map optimization modules.

GeomFMNet\cite{donati2020deep} extracts shape features directly from vertex coordinates using KPConv\cite{Thomas2019}, optimizing the functional map with a quadratic regularization term by solving multiple linear systems. DiffusionNet\cite{Sharp2020}, a state-of-the-art feature extractor capable of producing discretization-resistant and orientation-aware shape features, has been utilized by numerous follow-up works \cite{attaiki2021dpfm,donati2022deep,Donati2022,li2022learning,cao2022unsupervised,Sun_2023_ICCV,Cao2023,HU2023101189}. \textit{Donati et al.} \cite{donati2022deep} integrated the complex functional map\cite{Donati2022} into the deep functional map framework to address symmetry challenges in shape matching. AttentiveFMaps~\cite{li2022learning} combined functional maps at different spectral resolutions using a spectral attention mechanism. RFMNet\cite{HU2023101189} introduced a novel approach to functional map optimization incorporating wavelet preservation\cite{Hu2021}. However, its wavelets cannot be optimized for specific tasks due to strict requirements for wavelet-like or tight frame functions, which are often unattainable. \textit{Cao et al.} \cite{Cao2023} proposed a learning-based approach for robust shape matching, achieving superior performance under various settings through its refinement strategy. However, it introduces significant additional computational complexity. Notably, we demonstrate that the coupling loss function introduced in \cite{Cao2023} is a special case of our proposed unsupervised loss.

Despite the rapid advancements in deep functional map methods, extracting task-specific frequency information in a learning-driven manner to achieve desirable properties in functional map estimation remains an overlooked challenge. To address this gap, we propose a novel unsupervised learning-based architecture that leverages our flexible and powerful constraints, achieving state-of-the-art performance in a variety of challenging shape-matching scenarios.
\section{Background and notation}
Given a pair of non-rigid shapes  denoted as ${\mathcal{M}}$ and ${\mathcal{N}}$ respectively,  
let $ T: \mathcal M \to \mathcal N $ be a pointwise map between them, then the corresponding functional map $T_F$ is a linear transformation taking functions on ${\mathcal{N}}$ to functions on ${\mathcal{M}}$. Namely, given a function $ g\in\mathcal{L}^{2}(\mathcal{N})$, we define its map $ f\in\mathcal{L}^{2}(\mathcal{M})$ satisfying $ f=T_{F}(g)=g \circ T $. If we use the truncated eigenfunctions (usually the first $k$ ) $ \left\{\phi_{i}^{\mathcal{M}}\right\}_{i \geq 0} $ and $ \left\{\phi_{j}^{\mathcal{N}}\right\}_{j \geq 0} $ of the Laplace-Beltrami operators (LBOs)  defined in each shape as the basis to represent the functions, then the functional map $T_{F}$ can be expressed as a $k \times k$ matrix $\mathbf C$, which could transfer the basis coefficients of the functions between the shapes.

In discrete settings, the shape $\mathcal{M} $ and shape $\mathcal{N} $ are typically represented as two triangle meshes, with $n_{\mathcal{M}}$ and $n_{\mathcal{N}}$ vertices respectively. Then the function $f$ in 
$ \mathcal{L}^{2}(\mathcal{N})$ is discretized to a  vector $ \mathbf{f} \in \mathbb{R}^{n_{\mathcal{N}}}$.  According to the standard cotangent weight scheme \cite{meyer2003discrete}, the Laplace-Beltrami operators (LBOs) defined on them can be represented as Laplacian matrices  $ \mathbf{L}_{\mathcal M} $ and $ \mathbf{L}_{\mathcal N} $, where $ \mathbf{L}_{\mathcal M} = \mathbf{A}_{ \mathcal M}^{-1} \mathbf{B}_{ \mathcal M} \in \mathbb{R}^{n_{\mathcal{M}} \times n_{\mathcal{M}}}$, $ \mathbf{L}_{\mathcal N} = \mathbf{A}_{ \mathcal N}^{-1} \mathbf{B}_{ \mathcal N} \in \mathbb{R}^{n_{\mathcal{N}} \times n_{\mathcal{N}}}$ respectively. Here the matrix $\mathbf{A}$  is the diagonal matrix of lumped area elements and $\mathbf{B}$ is the cotangent weight matrix. We make use of the basis consisting of the first $k$ eigenfunctions of the Laplacian matrix and encode it in a matrix  $ \Phi_{\mathcal M} = \left[\phi_{1}^{\mathcal M}, \phi_{2}^{\mathcal M}, \dots, \phi_{k}^{\mathcal M} \right] \in \mathbb{R}^{n_{\mathcal{M}} \times k} $ having the eigenfunctions as its columns. We also encode the first $k$ eigenvalues of the Laplacian matrix as  a diagonal matrix
$ \Lambda_{\mathcal M} = \mathrm{diag}\{\lambda_1, \lambda_2, \dots, \lambda_{k} \} \in \mathbb{R}^{k \times k}$ with the eigenvalues as its diagonal elements.

Now the pointwise map $ T : \mathcal M \to \mathcal N $ can also be expressed as a matrix $ \Pi_\mathcal{MN} \in \mathbb{R}^{n_\mathcal{M} \times n_\mathcal{N}} $, s.t. $ \Pi_\mathcal{MN}(i, j) = 1 $, if  $ T(i) = j $ and $0$ otherwise, where $ i $ and $ j $ denote the vertex indices of shapes $ \mathcal M $ and $ \mathcal N $, respectively. Now, the map image $\mathbf{f}$ of $\mathbf{g}$ can be represented as $\mathbf{f}=\Pi_\mathcal{MN}\mathbf{g}$. In matrix notation, $\mathbf{C}_{\mathcal{NM}}$ is given by the projection of $ \Pi_\mathcal{MN} \in \mathbb{R}^{n_\mathcal{M} \times n_\mathcal{N}} $ onto the corresponding functional basis, i.e., $\mathbf{C}_{\mathcal{NM}}=\Phi_{\mathcal{M}}^{\dagger}\Pi_\mathcal{MN} \Phi_{\mathcal{N}}$, where $ \Phi^{\dagger}_{\mathcal{M}}=\Phi^{\mathrm{T}}_{\mathcal{M}} \mathbf A_{\mathcal{M}} $ is the Moore-Penrose pseudo-inverse of $ \Phi_{\mathcal{M}}$.

Generally, the computation of the functional map $\mathbf{C}_{\mathcal{NM}}$ mainly resorts to the constraints of descriptor preservation. Namely,
given $d$-dimensional features  $\mathbf D_\mathcal{M} \in \mathbb{R}^{n_{\mathcal{M}} \times d} $ and $\mathbf D_\mathcal{N} \in \mathbb{R}^{n_{\mathcal{N}} \times d} $ computed on each shape, 
they will be approximately preserved by the mapping
$T$. As the functional map $\mathbf{C}_{\mathcal{NM}}$ transfers the basis coefficients of the functions between the 
shapes, this allows the functional map $\mathbf{C}_{\mathcal{NM}}$  satisfying the system of linear equations
 $\mathbf{C}_{\mathcal{NM}} \Phi_{\mathcal{N}}^{\dagger} \mathbf D_{\mathcal{N}} = \Phi_{\mathcal{M}}^{\dagger} \mathbf D_{\mathcal{M}} $. In practice, this equation often couples with some penalization terms, i.e.,
\begin{small}
\begin{equation}\label{equ: desc and reg}
% \begin{split}
   \mathbf{C}_{\mathcal{NM}} = \mathop{\arg\min}\limits_{\mathbf{C}\!_{\mathcal{NM}}\!} \left\|  \Phi_{\mathcal{M}}\!^{\dagger} \mathbf D_{\mathcal{M}}\! - \mathbf{C}_{\mathcal{NM}}\! \Phi_{\mathcal{N}}\!^{\dagger} \mathbf D_{\mathcal{N}}\! \right\|^{2}_{\mathrm{F}} 
    + \lambda E_{reg}(\mathbf{C}_{\mathcal{NM}}\!), 
% \end{split}
\end{equation}
\end{small}
where $\left \| \cdot \right \|_{\mathrm{F}}$ denotes Frobenius norm.

As a last step, the estimated map $ \mathbf{C}_{\mathcal{NM}}$ can be converted to a point-to-point map commonly by nearest neighbor search between the aligned spectral embeddings, namely, $\Pi_{\mathcal{MN}} = NNsearch(\Phi_\mathcal{N}\mathbf{C}_{\mathcal{NM}}^\mathrm{T},\Phi_\mathcal{M}) $.
\section{Functional maps with spectral filter operator preservation}
The preservation constraints are critical in the functional map framework, yet methods that ensure pointwise maps maintain frequency awareness—i.e., the ability to adaptively capture important frequency information for functional map estimation—remain underexplored. To address this gap, we introduce a novel constraint called Spectral Filter Operator Preservation, which comes with rigorous theoretical guarantees and facilitates the computation of functional maps based on this preservation. Notably, the filter functions within this framework can be represented by arbitrary functions, allowing the constraint to flexibly emphasize favorable frequency information by constructing a set of filter functions. This capability establishes a strong foundation for subsequent learning-based frameworks. Furthermore, we demonstrate that existing approaches, such as heat kernel preservation~\cite{vestner2017kernelmatching}, ZoomOut~\cite{Ren2019}, and wavelet preservation~\cite{Hu2021}, are special cases of our proposed method.

\subsection{Spectral filter operator preservation}\label{sec: Method}
To encode the frequency information, we first define a spectral filter operator on the shape $\mathcal{N}$ as $R: \mathcal{L}^{2}(\mathcal{N}) \to  \mathcal{L}^{2}(\mathcal{N})$, which enhances or suppresses specific frequency components of the input functions (or signals). From the perspective of spectral signal processing, the ordered eigenfunctions and eigenvalues of the Laplace-Beltrami operators (LBOs) serve as analogs to the Fourier basis and frequencies. This allows for arbitrary filtering of functions by designing appropriate spectral filters. For simplicity, the discussion is confined to the discrete setting.

Given a function $ \mathbf{f} $ defined on the shape \(\mathcal{N}\), and a filter function $h(\lambda)$, the spectral filter operator $\mathbf{R}$ (which can also be represented as a matrix)  applied to the signal $ \mathbf{f} $  is expressed as:
\begin{equation}\label{eq: Fourier filter}
	\mathbf{R}\mathbf{f} = \Phi h(\Lambda) \Phi^{\dagger} \mathbf{f},   
\end{equation}
where $ h(\Lambda)= \mathrm{diag}\left\{h(\lambda_1), h(\lambda_2), \dots, h(\lambda_{k}) \right\} $. This formulation adjusts the Fourier coefficients of $\mathbf{f}$ (i.e., $\Phi^{\dagger}\mathbf{f}$) using the filter $h(\Lambda)$, represented as $h(\Lambda) \Phi^{\dagger} \mathbf{f}$. The filtered function is then obtained by applying the inverse Fourier transform.

The choice of the filter function $h(\lambda)$ is crucial in the spectral filter operator, as it directly determines the frequency information extracted from the input signal. Various filter functions have been proposed in the field of shape analysis to suit different tasks, such as heat diffusion filters~\cite{sun2009concise}, wave kernel filters~\cite{Aubry2011The}, and others, each tailored to extract specific frequency information. A optimising approach to optimize this process is data-driven design, which motivates our work on developing learnable preservation constraints.

\newtheorem{thm}{Remark}[section]
\begin{thm}\label{pro: remark_1}
Let $T: \mathcal{M} \to \mathcal{N}$ be an \textit{isometric} pointwise map, which induces a functional map $T_F: \mathcal{L}^{2}(\mathcal{N}) \to  \mathcal{L}^{2}(\mathcal{M})$. Denote the spectral filter operators on the shapes $\mathcal{M}$ and $\mathcal{N}$ as $R_\mathcal{M}$ and $R_\mathcal{N}$ respectively, where both operators share the same filter function $h$ applied to the eigenvalues of their corresponding Laplace-Beltrami operators (LBOs). Under these conditions, the spectral filter operators are preserved by the corresponding functional map, i.e.,
\begin{equation}
	T_F \circ R_\mathcal{N} = R_ \mathcal{M} \circ T_F.
\label{eq:operator_preservation}
\end{equation}
In the discrete setting, this preservation is represented as:
\begin{equation}
\mathbf{C}_{\mathcal{NM}}h(\Lambda_\mathcal{N}) = h(\Lambda_\mathcal{M})\mathbf{C}_{\mathcal{NM}}.
\label{eq:matrix of operator preservation}
\end{equation}
\end{thm}
\begin{proof}
The Eq.(\ref{eq:operator_preservation}) holds because the spectral filter operators are intrinsic. The invariance of the spectral filter operators under isometric deformations is a direct consequence of the invariance of the Laplace-Beltrami operator. For the discrete representation in Eq.(\ref{eq:matrix of operator preservation}), we know that $T_F$ can be represented as $T_F=\Pi_{\mathcal{MN}} \approx \Phi_\mathcal{M}\mathbf{C}_\mathcal{NM}\Phi_{\mathcal{N}}^\dagger$ (where the equality holds in the full basis). Substituting this into Eq.(\ref{eq:operator_preservation}), we get
\begin{equation*}
\Phi_\mathcal{M}\mathbf{C}_\mathcal{NM}\Phi_{\mathcal{N}}^\dagger \Phi_\mathcal{N}h(\Lambda_\mathcal{N})\Phi_\mathcal{N}^\dagger=\Phi_\mathcal{M}h(\Lambda_\mathcal{M})\Phi_\mathcal{M}^\dagger\Phi_\mathcal{M}\mathbf{C}_\mathcal{NM}\Phi_{\mathcal{N}}^\dagger.
\end{equation*}
Since $\Phi^{\dagger}\Phi=\mathbf{I}$, this simplifies to:
\begin{equation*}
\Phi_\mathcal{M}\mathbf{C}_\mathcal{NM}h(\Lambda_\mathcal{N})\Phi_\mathcal{N}^\dagger=\Phi_\mathcal{M}h(\Lambda_\mathcal{M})\mathbf{C}_\mathcal{NM}\Phi_{\mathcal{N}}^\dagger.
\end{equation*}
Multiplying both sides of the equation by $\Phi_\mathcal{M}^\dagger$ on the left and $\Phi_\mathcal{N}$ on the right, we obtain:
\begin{equation*}
\mathbf{C}_{\mathcal{NM}}h(\Lambda_\mathcal{N}) = h(\Lambda_\mathcal{M})\mathbf{C}_{\mathcal{NM}}.
\end{equation*}
This completes the proof. 
\end{proof}	

From the above observation, our spectral filter operator preservation can be interpreted as applying filter functions in the spectral domain and exchanging them through functional maps. Two key observations emphasize the significance of this approach:
\begin{itemize}
    \item The proposed preservation can be viewed as a \textit{generalized Laplacian commutativity}, as the standard Laplacian commutativity, $\mathbf{C}_{\mathcal{NM}}\Lambda_\mathcal{N} = \Lambda_\mathcal{M}\mathbf{C}_{\mathcal{NM}}$, is a special case achieved by setting $h(\Lambda)= \Lambda$ in Eq. \eqref{eq:matrix of operator preservation}.
    \item The filter functions $h$ can be defined as arbitrary functions, highlighting the flexibility of our constraint. This allows for the encoding of desired frequency information by learning adjustable filter functions, rather than relying solely on fixed spectral properties, as in standard Laplacian commutativity.
\end{itemize}

\begin{figure*}[hbt]
\centering
    \subfigure[Heat kernel preservation~\cite{vestner2017kernelmatching}]{
            \includegraphics[width=0.23\linewidth]{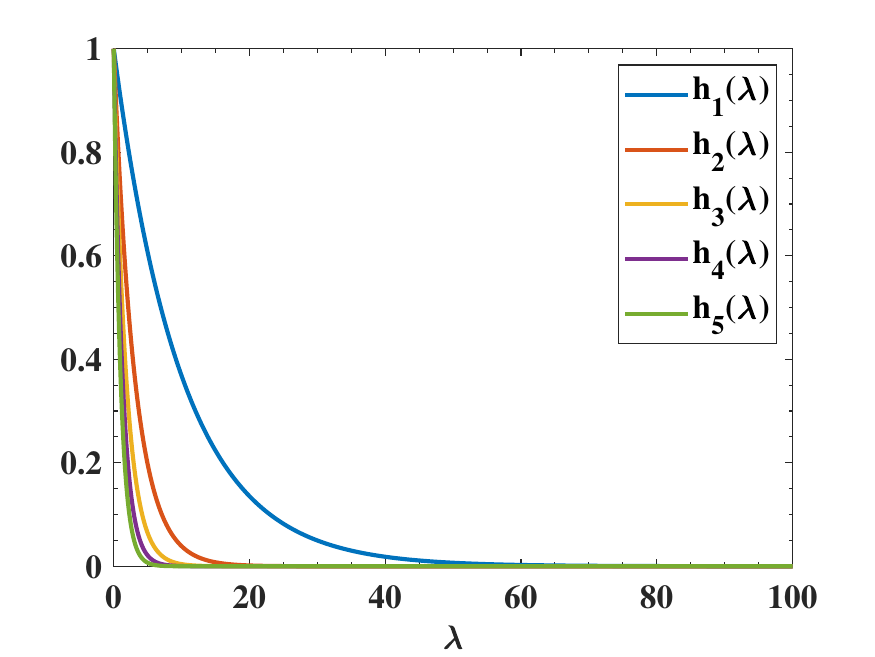}
            }
    \subfigure[ZoomOut~\cite{Ren2019}]{
            \includegraphics[width=0.23\linewidth]{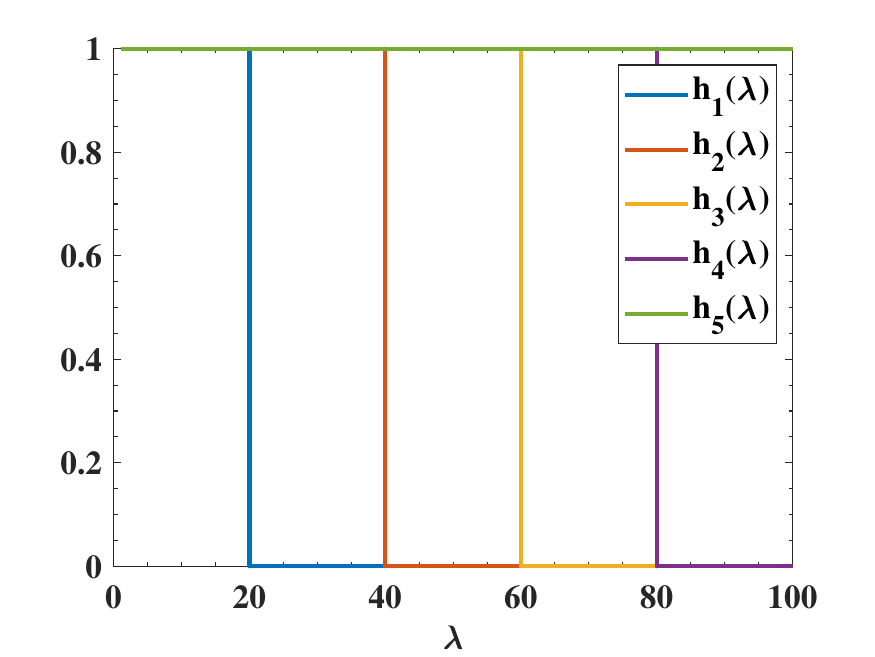}
            }
    \subfigure[Wavelet preservation~\cite{Hu2021}]{
            \includegraphics[width=0.23\linewidth]{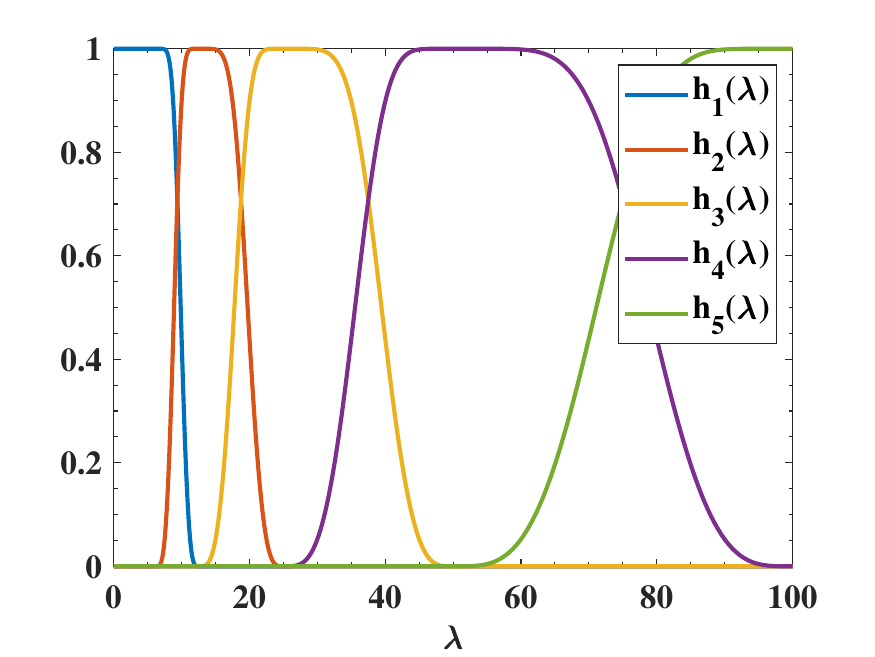}
            }
    \subfigure[Ours]{
            \includegraphics[width=0.23\linewidth]{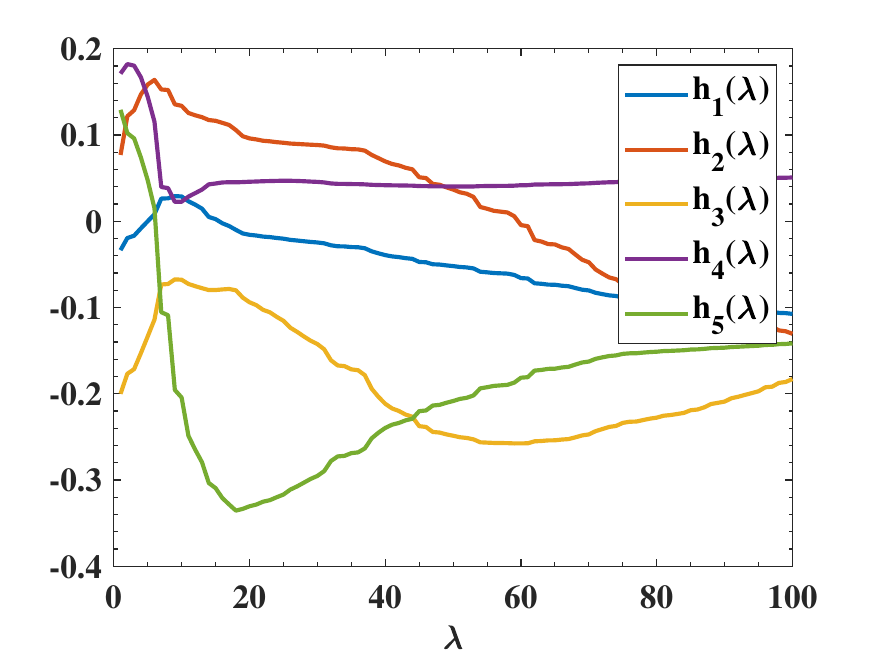}
            }
\caption{The filter functions of different methods. (a) The filter functions in heat kernel preservation are low-pass, focusing on low frequency information only. (b) The filter functions of ZoomOut are upsampled ideal filters, using to upsample resolutions of eigenfunctions (c) Meyer wavelet functions consist of a low-pass and a set of band-pass in wavelet preservation are tight frames strictly (d) Our learned filter functions are optimized on non-isometric datasets SMAL.}
\label{fig: filters of different methods}
\end{figure*}

To further improve the properness \cite{ren2021discrete} of the functional map ,  we replace one term $\mathbf{C}_\mathcal{NM}$ in Eq. \eqref{eq:matrix of operator preservation} with $\mathbf{C}_\mathcal{NM}^\Pi=\Phi_\mathcal{M}^\dagger\Pi_\mathcal{MN}\Phi_\mathcal{N}$. This substitution yields the following equation: 
\begin{equation}
\mathbf{C}_\mathcal{NM}h(\Lambda_\mathcal{N})=h(\Lambda_\mathcal{M})\mathbf{C}_\mathcal{NM}^\Pi.
\label{eq:couple_f_p}
\end{equation}

In the next section, we illustrate how this critical constraint can be applied to compute both the functional map and the pointwise map.

\subsection{Compute maps with spectral filter operator preservation}

A single spectral filter operator preservation, as described in Eq.(\eqref{eq:couple_f_p}) may not fully capture the frequency features of the maps, as each operator processes information from only a specific frequency band. To address this issue, we employ the preservation of multiple spectral filter operators as constraints to compute the functional map $\mathbf{C}_{\mathcal{NM}}$ and the corresponding pointwise map $\Pi_{\mathcal{MN}}$. 

Given a set of filter functions $\{h_s(\lambda)\}_{s=1}^S$, and leveraging Eq.(\ref{eq:matrix of operator preservation}), we formulate the computation of these maps as the following optimization problem:

\begin{equation}\label{equ: compute Pi and C with muti-channel}
\begin{aligned}
\min _{{\Pi}_\mathcal{MN}, \mathbf{C}_\mathcal{NM}} \left\|\mathbf{C}_\mathcal{NM}  H\left( \Lambda_{\mathcal{N}}\right) - H(\Lambda_{\mathcal{M}}) *\mathbf{C}_\mathcal{NM}^\Pi\right\|_{\mathrm{F}}^{2}, \\
	\text { s.t. } \quad \Pi_\mathcal{MN} \mathbf{1}=\mathbf{1}, {\Pi}_\mathcal{MN}^{\mathrm{T}} \mathbf{1} \leq \mathbf{1}, 
\end{aligned}
\end{equation}
where the filter function matrix $H(\Lambda)$$ = $$[ h_1(\Lambda)$ $||h_2(\Lambda)$ $||\cdots|| $ $h_S(\Lambda)]$ $\in$ $ \mathbb{R}^{k\times(S\cdot k)}$, $||$ denotes matrix connection, and $H(\Lambda_{\mathcal{M}})*\mathbf{C}_\mathcal{NM}^\Pi$ $ = \left[ h_1(\Lambda_{\mathcal{M}})\mathbf{C}_\mathcal{NM}^\Pi||\cdots||h_S(\Lambda_{\mathcal{M}})\mathbf{C}_\mathcal{NM}^\Pi\right].$

Given the pointwise map matrix $ \Pi_\mathcal{MN} $, then the Eq.\eqref{equ: compute Pi and C with muti-channel} is transformed to the following problem
\begin{equation}\label{equ: compute C by fixing Pi}	
	\mathbf{C}^{*}_\mathcal{NM} =\arg\min _{\mathbf{C}_\mathcal{NM}} \left\|\mathbf{C}_\mathcal{NM} H\left( \Lambda_{\mathcal{N}}\right) - H\left( \Lambda_{\mathcal{M}}\right)*{\mathbf{C}_\mathcal{NM}^\Pi}\right\|_{\mathrm{F}}^{2}. 
\end{equation}
Next, we efficiently compute $\mathbf{C}^{*}_\mathcal{NM}$ via the following important observation. 

\newtheorem{thm1}{Remark}[section]
\begin{thm}\label{pro: remark_2}
	~If the set of filter functions $\{h_s(\lambda)\}_{s=1}^S$ satisfy the condition that 
	$\sum_{s} h_s^{2}\left(\lambda\right) \neq 0, \forall \lambda,$
then the functional map in Eq.(\ref{equ: compute C by fixing Pi}) can be obtained via

\begin{equation}\label{MCFP compute C}
\begin{aligned}
    \mathbf{C}_\mathcal{NM} =  \left(\sum_{s}  h_s\left(\Lambda_{\mathcal{M}}\right) \mathbf{C}_\mathcal{NM}^\Pi  h_s\left( \Lambda_{\mathcal{N}}\right)\right) {\left(G^{-1}(\Lambda_{\mathcal{N}})\right)}. 
 \end{aligned}
\end{equation}
where $G(\Lambda_{\mathcal{N}})=\sum_{s}h_s^{2}\left(\Lambda_{\mathcal{N}}\right)$.
\begin{proof}
	~See Appendix \ref{app: Pro}. 
\end{proof}			
\end{thm}
Interestingly, Eq.~\eqref{MCFP compute C} aligns with the intuition of denoising, as it effectively uses multiple filters to refine the coarse functional map $ \mathbf{C}_\mathcal{NM}^\Pi$. Furthermore, the pointwise map can be obtained through a nearest-neighbor search. This process can also be extended to iteratively update both the functional map and the pointwise map, similar to the approach proposed in \cite{Ovsjanikov2012}.

\begin{table}[h!t]
\centering
\caption{Special cases of our preservation }
\scalebox{1.1}{
\begin{tabular}{lcc}
\hline
               & \multicolumn{2}{c}{Filters settings}  \\ \hline
Laplacian commutativity\cite{Ovsjanikov2012}  &\multicolumn{2}{c}{$H(\Lambda) = \Lambda$}     \\ %\hline         
Heat kernel preservation~\cite{vestner2017kernelmatching}           & \multicolumn{2}{c}{$H(\Lambda) = \exp(-t\Lambda)$}   \\ %\hline
ZoomOut~\cite{Ren2019}           & \multicolumn{2}{c}{Upsampled
ideal filters}  \\
Wavelet preservation~\cite{Hu2021} & \multicolumn{2}{c}{Meyer functions}    \\ %\hline

Ours        & \multicolumn{2}{c}{Arbitrary filters}   \\ \hline
\end{tabular}}
\label{tab: Rtt}
\end{table}

\subsection{Relation to other techniques} 
In addition to Laplacian commutativity, we emphasize that related works such as heat kernel preservation~\cite{vestner2017kernelmatching}, ZoomOut~\cite{Ren2019}, and wavelet preservation~\cite{Hu2021} can also be viewed as special cases of our framework. For a detailed comparison, see Table~\ref{tab: Rtt}.

\textbf{Heat kernel preservation.} 
Heat kernel preservation~\cite{vestner2017kernelmatching} formulates an optimization problem by leveraging the positive-definite heat kernel to construct pairwise descriptors. In this framework, functional maps are treated as a low-pass approximation of the permutation matrix in the truncated Laplacian eigenfunctions. The alternating heat diffusion between shapes can be interpreted as applying a set of low-pass filters to the functional map matrix, where the filter function is defined as $h(\Lambda) = \exp(-t\Lambda)$ in Eq.\eqref{MCFP compute C} (see Fig.\ref{fig: filters of different methods} (a)). However, the heat kernel is limited in its ability to encode informative features, as it neglects band-pass and high-pass frequency information.

\textbf{ZoomOut.} ZoomOut introduced an iterative spectral upsampling technique to align the functional maps and pointwise map from low to high frequencies, which can be considered as using a set of upsampled ideal filters to represent $\left\{h_s\left( \lambda\right)\right\}_{s=1}^{S}$ in Eq.~\eqref{equ: compute Pi and C with muti-channel}. For instance, setting $h_1\left( \Lambda\right) = \mathrm{diag} \{ \overbrace{1,\cdots,1}^{k_1},0,\cdots,0 \}$, $h_2\left( \Lambda\right) = \mathrm{diag}\{\overbrace{1,1,\cdots,1}^{k_2},0,\cdots,0\}$, $\cdots$, $h_S\left( \Lambda\right) = \mathrm{diag}\left\{1,1,\cdots,1,1,\cdots,1 \right\}$, where $k_1 < k_2< \cdots < k$, see Fig.~\ref{fig: filters of different methods} (b). 
In fact, a set of upsampled ideal filters are used to partition the eigenfunction matrix $\Phi_{\mathcal{M}} \in \mathbb{R}^{M \times k} $ into upsampled spectral resolutions sets $\left\{ \Phi_{\mathcal{M}, 1}\in{\mathbb{R}^{n_{\mathcal{M}} \times k_1}}, \Phi_{\mathcal{M}, 2}\in{\mathbb{R}^{n_{\mathcal{M}} \times k_2}}, \cdots, \Phi_{\mathcal{M}, S} \right\}$, where $\Phi_{\mathcal{M}, 1}$ consists of the first $k_1$ columns of the matrix $\Phi_{\mathcal{M}}$, $\Phi_{\mathcal{M}, S} = \Phi_{\mathcal{M}}$. By using those partitions $\left\{ \Phi_{\mathcal{M}, 1}, \Phi_{\mathcal{M}, 2}, \cdots, \Phi_{\mathcal{M, S}} \right\}$ to recover functional maps iteratively, i.e, ZoomOut. However, ZoomOut cannot enhance or suppress the frequency information softly since the value of the sampled ideal filter is either 0 or 1, leading to loss of information easily. 

\textbf{Wavelet preservation.} Multiple spectral manifold wavelets in wavelet preservation~\cite{Hu2021} are required to be preserved at each scale correspondingly. When we specify our filter functions  $\{h_s(\lambda)\}$ to be the multi-scale tight frame Meyer functions~\cite{leonardi2013tight} (see Fig.~\ref{fig: filters of different methods} (c)), our work will degrade to the wavelet preservation\cite{Hu2021}. However, wavelet functions in the preservation are unlearnable and can only encode fixed frequency features. 

From the above discussion, heat kernel preservation~\cite{vestner2017kernelmatching}, ZoomOut~\cite{Ren2019}, and wavelet preservation~\cite{Hu2021} can be regarded as special cases of our constraints. Each method can be interpreted as applying different filter functions applying to the functional maps,  i.e., lowpass filters for heat kernel preservation, upsampled ideal filters for ZoomOut and lowpass and a set of bandpass filters ~\cite{hammond2011wavelets} for wavelet preservation, respectively. These preservation constraints can be unified into a general framework, as described by Eq.~\eqref{equ: compute Pi and C with muti-channel}.

However, the aforementioned constraints are limited in their ability to construct adaptable filter functions for extracting frequency information, which significantly restricts their applicability. In contrast, our proposed constraint enables flexible selection of filter functions and even allows them to be learnable. This adaptability leads to more effective constraints for computing the maps. As shown in Fig.~\ref{fig: filters of different methods} (d), our filter functions are optimized for specific matching scenarios, encoding more relevant frequency features to enhance correspondence accuracy.

In the next section, we leverage the proposed constraint to develop a novel frequency-aware learning framework for shape matching.

\begin{figure*}[hbt]
	\centering
	\includegraphics[width=1\linewidth]{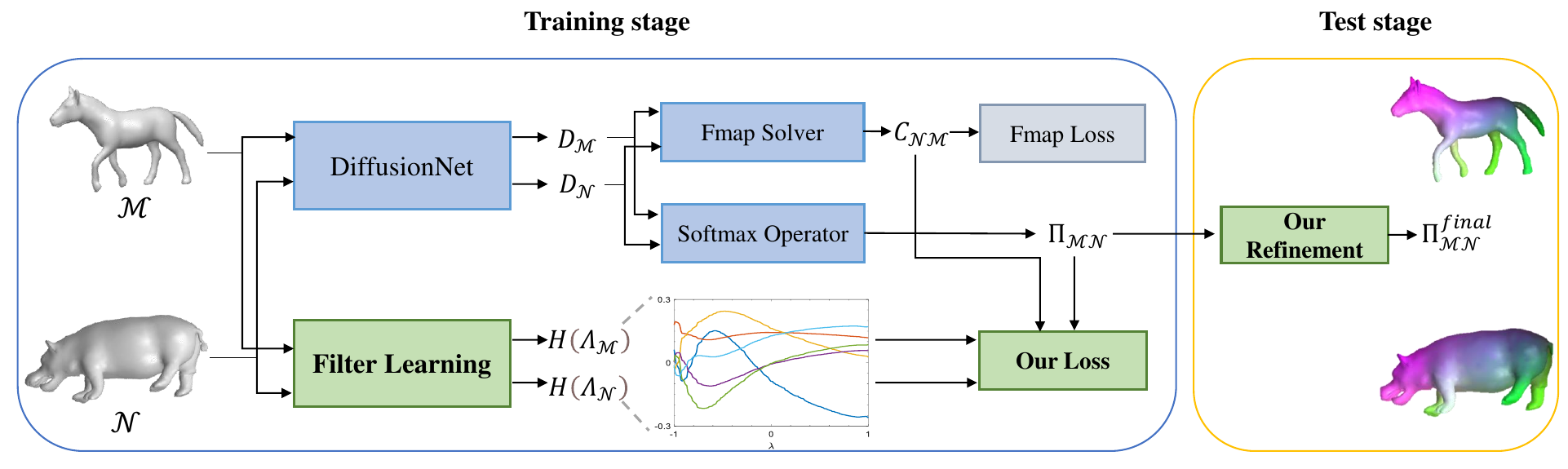}
	\caption{An overview of DFAFM. (1) Inputing a pair of shapes $\mathcal{M}$ and $\mathcal{N}$ to a trainable Siamese feature network to produce learned features $\mathbf{D}_{\mathcal{M}}$ and $\mathbf{D}_{\mathcal{N}}$ respectively. (2) Employing the filter learning layer to produce learned filter functions $H(\Lambda_{\mathcal{M}}) $ and $H(\Lambda_{\mathcal{N}}) $ respectively, i.e., Eq.~\eqref{eq: Jacobi combination}. (3) Utilizing learned features to compute the differentiable pointwise map $\Pi_{\mathcal{MN}}$ and functional map $\mathbf{C}_{\mathcal{NM}}$ by resorting to Softmax operator (Eq.~\eqref{eq: compute soft map}) and functional map solver (Eq.~\eqref{equ: desc and reg}) respectively. (4) A frequency awareness loss term (Eq.~\eqref{equ: frequency couple loss}) is used to supervise the functional map $\mathbf{C}_{\mathcal{NM}}$, pointwise map $\Pi_{\mathcal{MN}}$, as well as $H(\Lambda_{\mathcal{M}}) $ and $H(\Lambda_{\mathcal{N}}) $. (5) Using our refinement technique (Eq.~\eqref{MCFP compute C}) to refine a pointwise map $\Pi_{\mathcal{MN}}$, resulting in a more accurate and robust final correspondence $\Pi^{final}_{\mathcal{MN}}$ at the test stage.}
	\label{fig2: Our networks}
\end{figure*}

\section{Deep Frequency Awareness Functional Maps}
Building on the aforementioned spectral filter operator preservation constraints, we introduce a novel and efficient unsupervised deep functional map architecture, termed Deep Frequency Awareness Functional Maps (DFAFM). An overview of the pipeline is presented in Fig.~\ref{fig2: Our networks}.

\subsection{Learning input pointwise maps}\label{Learning input pointwise maps}

Unlike approaches such as \cite{Cuturi2013,eisenberger2020smooth}, which rely on the time-consuming optimal transport algorithm to compute the differentiable soft pointwise map $\Pi_{\mathcal{MN}}$ after extracting shape features, a more efficient strategy~\cite{Cao2023,Eisenberger2021} utilizes the softmax operator to generate a soft correspondence matrix. Specifically, the soft correspondence matrix is defined as:
\begin{equation}\label{eq: compute soft map}
    \Pi_{\mathcal{MN}} = \mathrm{Softmax}(\mathbf{D}_{\mathcal{M}} \mathbf{D}^{\mathrm{T}}_{\mathcal{N}}/\tau),
\end{equation}
where the element at position ($i$, $j$) represents the probability of correspondence between the $i$-th point on $\mathcal{M}$ and the $j$-th point on $\mathcal{N}$, and $\tau$ is the scaling factor to determine the softness of the correspondence matrix.

\subsection{Unsupervised loss}\label{sec: loss}

Most existing deep functional map methods focus solely on predicting optimized functional maps, often resulting in suboptimal performance. The coupling loss introduced by \textit{Cao et al.}~\cite{Cao2023} has demonstrated robust and accurate matchings by jointly considering functional and pointwise maps. Inspired by this insight, we propose a novel unsupervised coupling loss, termed Frequency Awareness Loss, based on our preservation framework. This loss jointly supervises the functional map, the pointwise map, and the training of filter functions. Specifically, the loss is defined as:
\begin{equation}\label{equ: frequency couple loss}
% \begin{aligned}
L_{freq}=\left\|\mathbf{C}_\mathcal{NM} H\left( \Lambda_{\mathcal{N}}\right) - H\left( \Lambda_{\mathcal{M}}\right) *\mathbf{C}^{\Pi}_\mathcal{NM}\right\|_{\mathrm{F}}^{2}, \\
% \end{aligned}
\end{equation}
Here, $\mathbf{C}^{\Pi}_{\mathcal{NM}} = \Phi_{\mathcal{M}}^{\dagger}\Pi_\mathcal{MN} \Phi_{\mathcal{N}}$, where $\Pi_\mathcal{MN}$ is the soft pointwise map matrix computed using Eq. \eqref{eq: compute soft map}, and $\mathbf{C}_\mathcal{NM}$ is obtained by solving Eq.~\eqref{equ: desc and reg}.

Our loss ensures that the functional map is associated with a valid pointwise map and optimizes them simultaneously under different frequency channels. Moreover, our loss can capture ideal spectral information comprehensively by learning filter functions, which leads to accurate and robust correspondence in more challenging matching settings.

Interestingly, by setting $ H(\Lambda) = \mathbf{I}, S=1$, our loss reduces to the coupling loss introduced by \textit{Cao et al.}~\cite{Cao2023}, which is defined as: 
\begin{equation}\label{equ: couple loss}
    L_{co} = \left\|\mathbf{C}_\mathcal{NM}-\mathbf{C}^{\Pi}_\mathcal{NM}\right\|^{2}_\mathrm{F}. 
\end{equation}
This demonstrates that the coupling loss is a special case of our proposed loss. However, the coupling loss lacks frequency awareness, as it cannot adaptively encode informative and important frequency components for specific tasks.

In addition, the common structural functional map regularisation~\cite{Ren2019} is used to penalize the bijectivity and orthogonality of the functional maps, i.e., 
\begin{equation}\label{equ: fmap loss}
    L_{fmap} = \theta_{bi} L_{bi} + \theta_{or} L_{or}, 
\end{equation}
where the bijectivity and orthogonality loss terms can be expressed as $L_{bi} =\left\| \mathbf{C}_\mathcal{NM} \mathbf{C}_\mathcal{MN} -\mathrm{I}\right\|^{2}_\mathrm{F}$ and $L_{or} =\left\| \mathbf{C}_\mathcal{NM} \mathbf{C}_\mathcal{NM}^{\mathrm{T}} -\mathrm{I}\right\|^{2}_\mathrm{F}$ respectively. 

To handle non-isometric matching, we also use the smoothness penalty on the pointwise map, based on the Dirichlet energy of shape vertices, i.e.,  
\begin{equation}\label{equ: smooth loss}
    L_{sm} = \left\|\Pi_\mathcal{MN} V_{\mathcal{N}}\right\|^{2}_{W_{\mathcal{N}}}, 
\end{equation}
where $W_{\mathcal{N}}$ is the cotangent weight matrix of shape ${\mathcal{N}}$.
At last, our final loss is  presented as the weighted combination of the above-stated loss functions, 
\begin{equation}\label{equ: total loss}
    L_{total} = \theta_{freq} L_{freq} + \theta_{fmap} L_{fmap} + \theta_{sm} L_{sm}. 
\end{equation}
Last but not least, we can change the direction of the pointwise map and functional map to build a bidirectional loss function, which can penalize them bi-directionally.

\subsection{Frequency awareness preservation with filter learning}\label{sec: MSFOP}
An important property is that our constraints can be optimized for different matching scenarios by learning filter functions, which can flexibly extract important frequency information for functional map estimation to promote desirable results. 

\textbf{Learnable spectral filter functions.}
The core of our idea is to use a linear combination of a set of basis functions $\left\{ g_l \right\}_{l=0}^{L} $ to represent each filter function and make all the representative coefficients learnable. Namely, let
\begin{equation} %\label{equ: filter_learnable_lambda}
    h_s\left(\lambda\right) = \sum_{l} w_{s l} g_l\left(\lambda\right), s=1,2,...,S.
\end{equation}
Encoding all representative coefficients of $S$ filter functions as a matrix $W = \left( w_{sl} \right) \in \mathbb {R}^{L \times S} $, with each filter function's coefficients as a column.

\textbf{Basis for spectral filter functions.} Various functions could be used to express our learnable filter functions. Here we choose the Jacobi polynomials to express them, as they have many attractive properties such as orthogonality, excellent capacity of representation, recursive computation, and so on. In particular, the Jacobi polynomials has a very general form, with several orthogonal polynomials as their special cases, such as the Chebyshev polynomials and the Legendre polynomials.
 
Now, we express our filter functions based on a set of Jacobi polynomials \{${J}^{a,b}_l({\lambda})\}_{l=0}^L $, i.e., 
\begin{equation}\label{eq: Jacobi combination}
    h_s({\lambda}) = \sum_{l} w_{s l} {J}^{a,b}_l({\lambda}),
\end{equation}
where $a$ and $b$ are also trainable parameters. For ease of reading, more details regarding the description of the Jacobi polynomials and how to integrate them into our network are shown in the Supplementary Material.

\subsection{Pointwise maps computation during inference}\label{sec: refinement}

To improve the accuracy and robustness of the final pointwise maps during the test stage, several post-processing techniques from axiomatic functional map methods~\cite{Ren2019,eisenberger2020smooth,Hu2021} are commonly used by existing deep functional map approaches~\cite{eisenberger2020deep,HU2023101189,donati2022deep}. Unlike previous deep functional map methods that directly output results without updating parameters during inference, UnsupRSFMNet\cite{Cao2023} introduced a refinement strategy called test-time adaptation. This strategy adjusts the parameters of the networks for each test pair individually through 15 backpropagation iterations, essentially training the networks on the test datasets. Consequently, its superior performance heavily relies on this extremely time-consuming refinement technique. In contrast, our approach, as shown in Eq.\eqref{MCFP compute C}, explicitly models the relationship between the functional map, pointwise maps, and a set of filter functions. This serves as a natural refinement technique, using the learned filter functions to refine the functional maps directly, eliminating the need for time-consuming post-processing.

During inference, we first directly obtains the pointwise map based on similarities of learned feature by using nearest neighbour search, i.e., 
\begin{equation}\label{equ: nnsearch in test}
    \Pi_{\mathcal{MN}} = NNsearch(\mathbf{D}_\mathcal{N},\mathbf{D}_\mathcal{M}).
\end{equation}

Note that the optimized filter functions $H(\Lambda_{\mathcal{M}})$ and $H(\Lambda_{\mathcal{N}})$ are supervised by frequency awareness loss term Eq.~\eqref{equ: frequency couple loss}, so the filter functions obviously retain desirable frequency information and ignore the noise frequency information. Then, with the computed pointwise map matrix $ \Pi_\mathcal{MN} $, as well as the learned filter functions $H(\Lambda_{\mathcal{M}})$ and $H(\Lambda_{\mathcal{N}})$, we spontaneously obtain the more accurate and robust functional maps $\mathbf{C}_{\mathcal{NM}}^\mathrm{*}$ via Eq.~\eqref{MCFP compute C}.  Finally, the final pointwise map is acquired by using nearest neighbour search again, i.e., $\Pi^{final}_{\mathcal{MN}} = NNsearch(\Phi_\mathcal{N}\mathbf{C}_{\mathcal{NM}}^\mathrm{*T},\Phi_\mathcal{M})$.

\section{Experiments and results}
\subsection{Implementation}\label{sec: Imp}

We use DiffusionNet as a feature extractor with its default settings, which uses 128-dimensional WKS~\cite{Aubry2011The} as input features and produces 256-dimensional learned features for the DFAFM network. We set the number of Jacobi polynomials orders $L=8$, the number of channels of spectral filters $S=6$, and the truncated eigensystems $k=200$. For pointwise map computation and the regularised functional map solver, we set $\tau=0.07$ in Eq.~\eqref{eq: compute soft map} and $\lambda = 100$ in Eq.~\eqref{equ: desc and reg}, respectively. In terms of our unsupervised loss, we empirically set $\theta_{bi} = 1$ and $\theta_{or} = 1$ in Eq.\eqref{equ: fmap loss}, $\theta_{freq} = 1$ and $\theta_{fmap} = 1$ in Eq.\eqref{equ: total loss} for near-isometric matching. In the context of non-isometric matching, the loss weight $\theta_{sm}$ for the smoothness penalty term is set to $5.0$. For training, we use the Adam optimizer \cite{Kingma15} with a learning rate of $0.001$ for all learning parameters. Note that, we use the mean geodesic error \cite{Kim2011BIM} to evaluate shape correspondence accuracy, which is computed over all pairs and points in the dataset and normalized by the geodesic diameter of the source shape.

\subsection{Baselines}
We extensively compare our method with existing non-rigid shape matching methods, which we categorize as follows:

\begin{itemize}
    \item \textit{Axiomatic approaches}, including BCICP \cite{ren2018continuous}, ZoomOut \cite{melzi2019matching}, Smooth Shells \cite{eisenberger2020smooth}, DiscreteOp \cite{ren2021discrete}, and MWP \cite{Hu2021}.
    \item \textit{Unsupervised approaches}, including Deep Shells \cite{eisenberger2020deep}, DUO-FMNet \cite{donati2022deep},  WTFMNet \cite{liu2022wtfm}, AttentiveFMaps \cite{li2022learning}, RFMNet \cite{HU2023101189} and UnsupRSFMNet \cite{Cao2023},
    where UnsupRSFMNet has disable the test-time-adaptation in the following experiments. Moreover, we compare with UnsupRSFMNet in the same scenario by using the strategy of test-time-adaptation. 
\end{itemize}

\begin{figure*}[h!t]
	\centering
	\includegraphics[width=.99\linewidth]{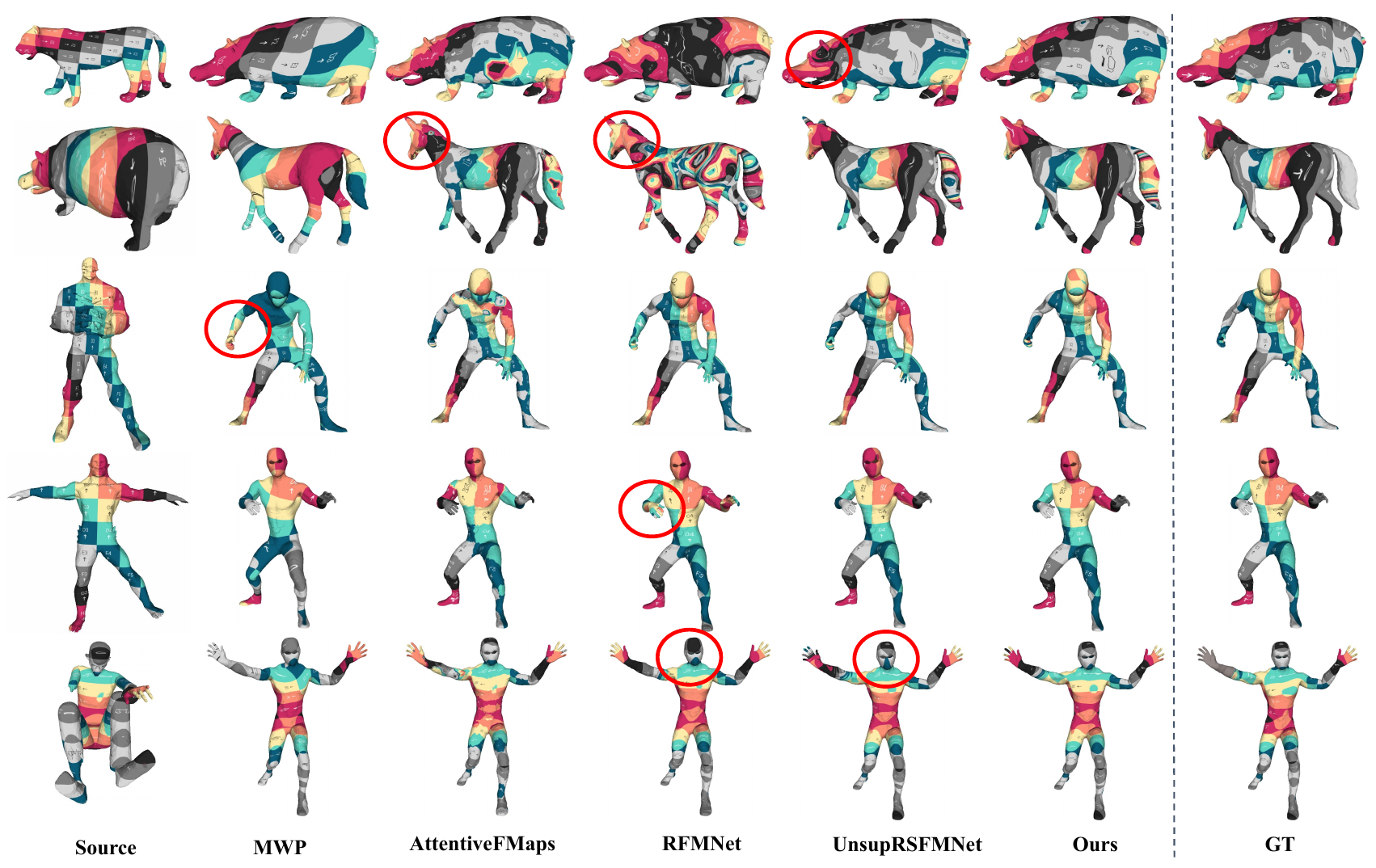}
	\caption{Comparisons with other methods on non-isometric shape matching. Matching results are visualized through texture transfer for shapes from the SMAL \cite{zuffi2017} and DT4D-H \cite{2022SmoothNonRigidShapeMatchingviaEffectiveDirichletEnergyOptimization} datasets. Compared to other methods, our approach results in fewer errors and less texture distortion, demonstrating its superior matching performance for non-isometric shapes. GT means ground-truth results.}
	\label{fig3: non_iso_compar}
\end{figure*}

\subsection{Results}
Extensive experimental results on plenty of datasets including challenging ones such as non-isometric and noisy datasets are used to evaluate our method with several state-of-the-art shape correspondence methods. All results are multiplied by 100 for the sake of readability. 

\begin{table}[h!t]
\centering
\caption{Benchmark tests on remeshed FAUST, SCAPE, and the robustness evaluation on anisotropic remeshed FAUST, SCAPE, respectively. The numbers in the table are mean geodesic errors ($\times 100$). \textbf{Bold}: Best, \underline{Underline}: Runner-up. TTA denotes using Test-Time Adaptation from UnsupRSFMNet for refinement at inference time \cite{Cao2023}.}
\scalebox{1}{
\begin{tabular}{lrrrr}
\hline
Train               & \multicolumn{2}{c}{F} & \multicolumn{2}{c}{S} \\ \cline{2-5} 
Test                & F        & F\_a       & S        & S\_a         \\ \hline
                    & \multicolumn{4}{c}{Axiomatic Methods}          \\
BCICP               & 6.1         & 8.5        & 11.0        & 14.0        \\
ZoomOut             & 6.1         & 8.7        & 7.5         & 15.0        \\
SmoothShells        & 2.5         & 5.4        & 4.7         & 5.0         \\
DiscreteOp          & 5.6         & 6.2        & 13.1        & 14.6        \\
MWP                 & 3.1         & 8.2        & 4.1         & 8.7            \\ \hline
                    & \multicolumn{4}{c}{Unsupervised Methods}                                                                     \\
Deep Shells         & 1.7         & 12.0       & 2.5         & 10.0         \\
DUO-FMNet           & 2.5         & 3.0        & 2.6         & 2.7          \\
WTFM                & 2.6         & 4.3        & 3.1         & 4.8          \\
AttentiveFMaps      & 1.9         & \underline{2.4}        & 2.2         & 2.3           \\
RFMNet              & 1.7         & 3.6        & 2.1         & 3.9           \\
UnsupRSFMNet        & \underline{1.6}         & 2.5        & \underline{1.9}         & \underline{1.9}      \\
Ours                & \textbf{1.6}         & \textbf{2.0}       & \textbf{1.9}         &\textbf{1.9}         \\ \hline
UnsupRSFMNet(+TTA)        & 1.6         & 1.9        & 1.9         & 1.9      \\
Ours(+TTA)        & \textbf{1.5}         & \textbf{1.8}       & \textbf{1.8}         & \textbf{1.8}      \\
\hline

\end{tabular}}
\label{tab: near-iso and aniso}
\end{table}

\textbf{Near-isometric shape matching}. 
We evaluate our method on the remeshed versions \cite{ren2018continuous} of the standard benchmarks FAUST and SCAPE (F and S for short, respectively), which are more challenging than the original datasets. The FAUST consists of 100 human shapes, which shows 10 different people in 10 different poses, is splited into 80/20 for training and testing. The SCAPE contains 71 human shapes, which shows the same person in different poses, is splited into 51/20 for training and testing. 

The results of these benchmarks are provided in Table \ref{tab: near-iso and aniso}, where our method is compared with current state-of-the-art axiomatic and unsupervised learning approaches. The results indicate that our method performs better than the previous state-of-the-art methods. 

\textbf{Matching with anisotropic meshing.}
To evaluate the robustness on different discretizations, we train networks on remeshed datasets and test them on anisotropic remeshed versions (denoted F\_a and S\_a, respectively), which have different mesh connectivity to the original datasets. 

Based on the results presented in Table \ref{tab: near-iso and aniso}, we can observe that our method demonstrates greater resilience to changes in triangulation compared to existing state-of-the-art methods \cite{Cao2023,li2022learning,HU2023101189}. Such as training on F and testing on F\_a, these methods suffer significant performance declines, tend to overfit mesh connectivity, resulting in inaccurate predictions. In contrast, our method exhibits stronger robustness to varying mesh connectivity and consistently surpasses the current state of the art. Besides, we also use the fine-tuning technique called Test-Time
Adaptation from UnsupRSFMNet\cite{Cao2023} for refinement at inference time, and the final result shows that our method is better than it.

\begin{table}[h!t]
\centering
\caption{ Cross-dataset generalization evaluation on anisotropic remeshed FAUST, SCAPE, respectively. The numbers in the table are mean geodesic errors ($\times 100$). \textbf{Bold}: Best, \underline{Underline}: Runner-up.}
\scalebox{1}{
\begin{tabular}{lrrrr}
\hline
Train               & \multicolumn{2}{c}{F} & \multicolumn{2}{c}{S} \\ \cline{2-5} 
Test                & S        & S\_a       & F        & F\_a         \\ \hline
                    & \multicolumn{4}{c}{Axiomatic Methods}          \\
BCICP           & 11.0        & 14.0       & 6.1         & 8.5             \\
ZoomOut         & 7.5         & 15.0       & 6.1         & 8.7             \\
SmoothShells        & 4.7         & 5.0        & 2.5         & 5.4         \\
DiscreteOp          & 13.1        & 14.6        & 5.6         & 6.2        \\
MWP           & 4.1         & 8.7                 & 3.1         & 8.2         \\ \hline
                    & \multicolumn{4}{c}{Unsupervised Methods}                                                                     \\
Deep Shells         & 5.4         & 16.0       & 2.7         & 15.0         \\
DUO-FMNet           & 4.2         & 4.4        & 2.7        & 3.1          \\
WTFM                & 4.1         & 4.6        &   2.9      &   4.8        \\
AttentiveFMaps      & \underline{2.6}         & \underline{2.8}        & 2.2        & \textbf{2.5}          \\
RFMNet              & \textbf{2.3}         & \textbf{2.6}        & \textbf{1.7}        & 3.6          \\
UnsupRSFMNet         & 6.7         & 8.9       &  4.8       &  7.0        \\
Ours                &  2.7    & 2.9      & \underline{1.9}       &  \underline{2.6}       \\ 
\hline
UnsupRSFMNet(+TTA)        & 2.2         & 2.4        & \textbf{1.6}         & 2.1      \\
Ours(+TTA)        & \textbf{2.2}         & \textbf{2.3}       & 1.8         & \textbf{2.1}      \\
\hline

\end{tabular}}
\label{tab: generalization}
\end{table}

\textbf{Cross-dataset generalization.}
The performance of generalization across datasets has been the focus of attention, and here, we set up four sets of experiments to evaluate the performance of our method in generalization across datasets, i.e., training on F and testing on S and S\_a, and vice versa. 
As shown in Table \ref{tab: generalization}, the quantitative results demonstrate our approach achieves comparable performance with state-of-the-art in most settings. Nevertheless, the existing cutting edge approach~\cite{Cao2023} suffers from huge performance drops when testing on the cross-dataset generalisation datasets, which demonstrates substantially its inadequate generalization ability compared to existing learning-based methods. Additionally, the fine-tuning matching results are summarized in the bottom of the Table \ref{tab: generalization}. Note that our approach is superior to it in most settings. Obviously, the results also show that the superior performance of UnsupRSFMNet heavily relies on the test-time adaptation process, our method is more robust than it. 

\begin{table}[h!t]
\centering
\caption{Non-isometric matching on SMAL\cite{zuffi2017} and DT4D-H \cite{2022SmoothNonRigidShapeMatchingviaEffectiveDirichletEnergyOptimization}. The numbers in the table are mean geodesic errors ($\times 100$). \textbf{Bold}: Best, \underline{Underline}: Runner-up.}
\scalebox{1}{
\resizebox{\linewidth}{!}{
\begin{tabular}{lrrr}
\hline
\multirow{2}{*}{}   & \multicolumn{1}{c}{\multirow{2}{*}{SMAL}} & \multicolumn{2}{c}{DT4D-H} \\ \cline{3-4} 
                    & \multicolumn{1}{c}{}                      & intra-class  & inter-class \\ \hline
                    & Axiomatic Methods                         &              &             \\
ZoomOut             & 38.4                                      & 4.0          & 29.0        \\
SmoothShells        & 36.1                                   & 1.1          & 6.3         \\
DiscreteOp         &  38.1                                       & 3.6          & 27.6        \\
MWP                  & 22.3                                   &   1.7           &   25.4          \\ \hline
                    & Unsupervised Methods                      &              &             \\
Deep Shells         & 30.4                                    & 3.4          & 31.1        \\
DUO-FMNet           & 32.8                                   & 2.6          & 15.8        \\
WTFM                & 24.8                                      & 3.9             & 41.0            \\
AttentiveFMaps      & \underline{5.4}                                      & 1.7          & 11.6        \\
RFMNet               & 20.3                                    & 1.5          & 13.9            \\
UnsupRSFMNet         & 5.5                                       & \underline{0.9}          & \underline{5.2}         \\
Ours                & \textbf{4.3}                                       & \textbf{0.9}          & \textbf{4.2}         \\ \hline
UnsupRSFMNet(+TTA)        & 3.9         & 0.9        & 4.1          \\
Ours(+TTA)        & \textbf{3.8}        & \textbf{0.9}       & \textbf{3.9}           \\
\hline

\end{tabular}}}
\label{tab: non-iso}
\end{table}

\textbf{Non-isometric shape matching.}
In the examination of non-isometric shape matching, our approach undergoes rigorous evaluation across two distinct datasets: SMAL \cite{zuffi2017} and DT4D-H \cite{2022SmoothNonRigidShapeMatchingviaEffectiveDirichletEnergyOptimization}. The SMAL dataset is comprised of 49 shapes representing four-legged animals across eight species, partitioned into a training set of 29 instances and a testing set of 20 instances. On the other hand, DT4D-H encompasses nine classes of humanoid shapes, with a training-testing split of 198 and 95 instances, respectively. This bifurcated analysis serves to comprehensively assess the robustness and efficacy of our method in non-isometric shape-matching scenarios.

As shown in Table \ref{tab: non-iso},  our method exhibits superior performance compared to the existing state-of-the-art approaches on the challenging SMAL dataset. In the context of intra-class matching on the DT4D-H dataset, our approach produces excellent matching results. For inter-class shape matching, our approach surpasses existing axiomatic and unsupervised methods by a substantial margin, emerging as the top-performing method in this category when benchmarked against the current state-of-the-art.  We also utilize the fine-tuning technique, and the quantitative results show that our method is still better than UnsupRSFMNet\cite{Cao2023}. These findings underscore the effectiveness and versatility of our proposed method across diverse challenges in non-isometric shape matching. Fig. \ref{fig3: non_iso_compar} illustrates some qualitative results of our method with comparisons to recent state-of-the-art methods on both SMAL \cite{zuffi2017} and DT4D-H \cite{2022SmoothNonRigidShapeMatchingviaEffectiveDirichletEnergyOptimization}  datasets. It observes that our method consistently outperforms existing approaches, and produces smooth and accurate matching results even in the presence of large non-isometric distortions.

\begin{table}[h!t]
\centering
\caption{Topological noise on TOPKIDS\cite{lahner2016shrec}. The table presents the mean geodesic errors ($\times 100$). \textbf{Bold}: Best, \underline{Underline}: Runner-up. Both RFMNet and Ours* employ residual multilayer perceptron (MLP) layers as feature extractors, with SHOT descriptors serving as input features, rather than using DiffusionNet.}
\scalebox{1.1}{
\begin{tabular}{lr}
\hline
                    & \multicolumn{1}{c}{TOPKIDS} \\ \hline
                    \multicolumn{2}{c}{Axiomatic Methods}    \\
ZoomOut             & 33.7                      \\
DiscreteOp          & 35.5                       \\
SmoothShells        & 11.8                      \\
MWP                 & 5.7                      \\ \hline

                    \multicolumn{2}{c}{Unsupervised Methods}     \\
SURFMNet            & 48.6                       \\
ConsistFMaps        & 39.3                     \\
Deep Shells         & 13.7                       \\
WTFM                & 28.2                        \\
AttentiveFMaps      & 23.4                       \\
% RFMNet              & 18.7             \\ 
UnsupRSFMNet        & \underline{9.4}            \\
Ours                & \textbf{6.3}   \\\hline

UnsupRSFMNet(+TTA)        & 9.2              \\
Ours(+TTA)        & \textbf{6.2}            \\
\hline
RFMNet               & \underline{4.9}    \\
Ours* &\textbf{3.0}\\ \hline

\end{tabular}}
\label{tab: topo}
\end{table}

\begin{figure*}[h!t]
	\centering
	\includegraphics[width=1\linewidth]{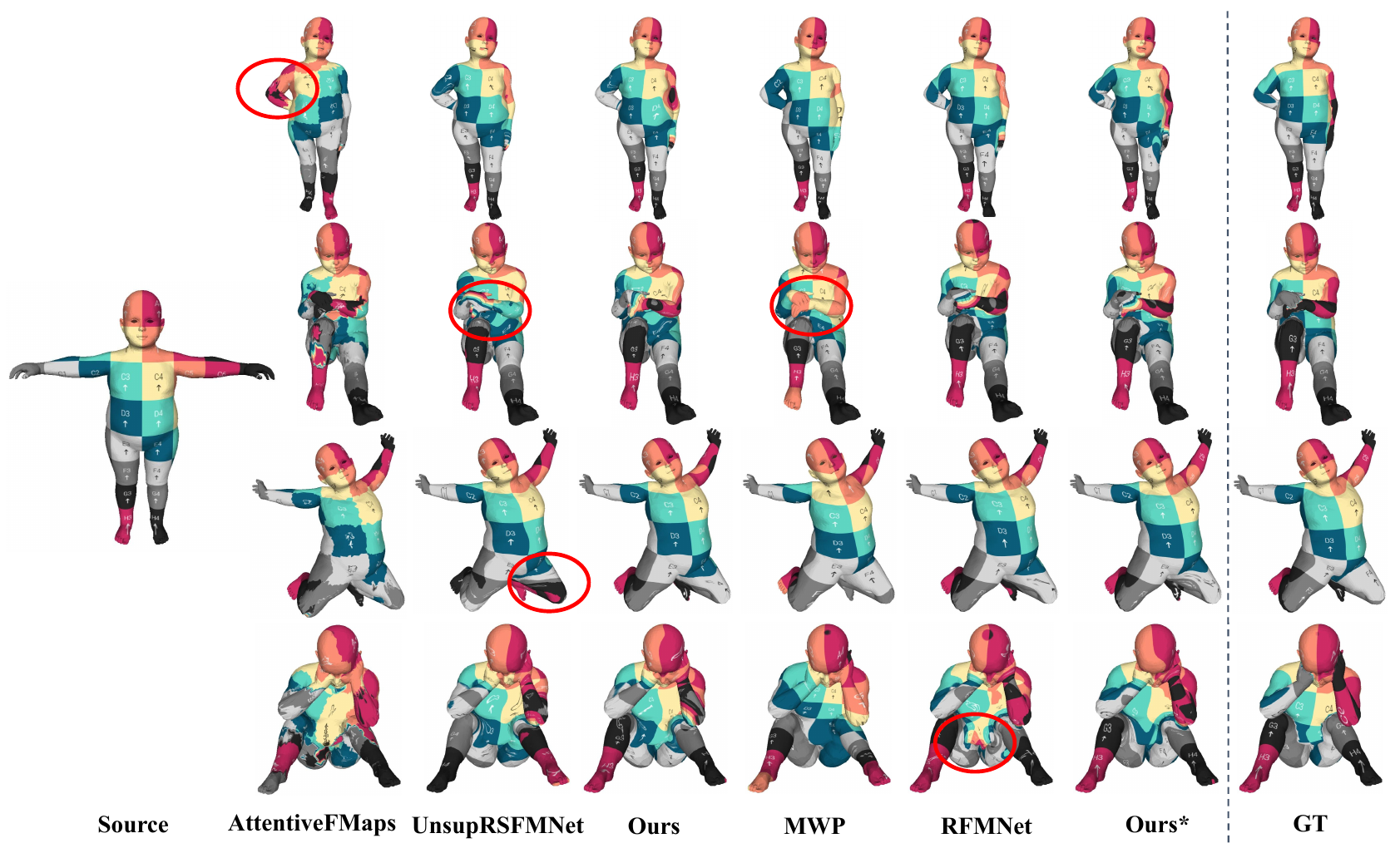}
	\caption{Comparisons with other methods on shape matching with topological noise, where shapes from SHREC’16 TOPKIDS benchmark \cite{lahner2016shrec}. Our results' smoother and more accurate texture distribution illustrates that our approach is more robust to topological noise than existing methods. GT means ground-truth results.}
	\label{fig3: top_compar}
\end{figure*}

\textbf{Matching with topological noise.}
Topological noise poses significant challenges for shape-matching methods because it distorts the intrinsic shape geometry in a non-isometric manner. To evaluate the robustness of our method against topological noise, we conducted experiments using the SHREC'16 TOPKIDS dataset, which consists of 25 shape pairs ($\approx 12\mathrm{K}$ vertices) undergoing near-isometric deformations with severe topological artifacts (e.g., touching hands creating geodesic shortcuts). The ground-truth correspondences were established by matching other shapes to a selected reference T-pose clean shape. 

The matching results are summarized in Table \ref{tab: topo}. Compared to the near-isometric shape matching task, the accuracy of all methods drops significantly when applied to shapes with topological noise. Nevertheless, our method with/without test-time adaptation achieves the best performance and is much more robust against topological noise. This is because, unlike UnsupRSFMNet or other approaches, our method incorporates frequency awareness, enabling it to capture critical information. 

However, it is important to note that the authors of DiffusionNet \cite{Sharp2020} have highlighted its limitations. DiffusionNet is designed to exploit the geometric structure of surfaces, and therefore it is not inherently robust to topological errors or outliers. Consequently, all methods using DiffusionNet as a feature extractor fail to achieve better results in this task. To address this issue, we adopt a meaningful strategy from RFMNet \cite{HU2023101189}, which integrates the input descriptor SHOT \cite{salti2014shot} with seven residual multilayer perceptron (MLP) layers as spatial feature extractors, referred to as Our*. This substitution significantly improves the performance of our method. Additionally, our method still outperforms RFMNet, which also employs SHOT+MLP as its feature extractor. Unlike RFMNet, which relies on fixed wavelet filters, our approach uses data-driven, learnable filters. These filters optimize functional maps by learning task-specific optimal filters, thereby yielding superior matching results. Meanwhile, a visual example is shown in Fig. \ref{fig3: top_compar} where our method provides smoother and more accurate texture transfer results than others.

\begin{table}[h!t]
\centering
\caption{Comparison to standard Laplacian commutativity loss (LC). Our method demonstrates several superior properties, including multiscale capabilities, learnability, and the ability to obtain proper functional maps. In contrast, the standard Laplacian commutativity loss operates at a single scale, maintains only fixed eigenvalues, and is unable to produce proper functional maps.}
\scalebox{1.0}{
\begin{tabular}{lcccccccccccc}
\hline
            & \multicolumn{2}{c}{Filter functions }   & \multicolumn{2}{c}{Multiscale }& \multicolumn{2}{c}{Learnabele} & \multicolumn{2}{c}{Proper} \\ \hline
LC & \multicolumn{2}{c}{$H(\Lambda) = \Lambda$ }   & \multicolumn{2}{c}{\ding{55}} & \multicolumn{2}{c}{\ding{55}} & \multicolumn{2}{c}{\ding{55}}   \\ %\hline 
Ours    & \multicolumn{2}{c}{Arbitrary function}   & \multicolumn{2}{c}{\ding{51}}  & \multicolumn{2}{c}{\ding{51}}  & \multicolumn{2}{c}{\ding{51}}   \\ \hline
\end{tabular}}
\label{tab: Laplacian commutativity V.S. Ours}
\end{table}

\begin{table*}[h!t]
\centering
\caption{Quantitative comparison with Laplacian commutativity couple loss on all datasets. The table reports the mean geodesic errors ($\times 100$), with more precise errors provided inside the parentheses.}
% \scalebox{1.0}{
\begin{tabular}{lrrrrrrrrrrrrr}
\hline
Train                                               & \multicolumn{4}{c}{F}                                                                                  & \multicolumn{1}{c}{} & \multicolumn{4}{c}{S}                                                                                  & \multicolumn{1}{c}{\multirow{2}{*}{SMAL}} & \multicolumn{2}{c}{DT4T-H}                                        & \multicolumn{1}{c}{\multirow{2}{*}{TOPKIDS}} \\ \cline{2-5} \cline{7-10} \cline{12-13}
Test                                                & \multicolumn{1}{c}{F} & \multicolumn{1}{c}{F\_a} & \multicolumn{1}{c}{S} & \multicolumn{1}{c}{S\_a} & \multicolumn{1}{c}{} & \multicolumn{1}{c}{F} & \multicolumn{1}{c}{F\_a} & \multicolumn{1}{c}{S} & \multicolumn{1}{c}{S\_a} & \multicolumn{1}{c}{}                         & \multicolumn{1}{c}{intra class} & \multicolumn{1}{c}{inter class} & \multicolumn{1}{c}{}                         \\ \hline
Laplacian couple                                       & 1.5                      & 4.0                      & 10.0                      & 14.4                      &                      & 11.2                      & 13.4                      & 1.9 (1.87)                       & 2.0                      & 8.3                                          & 1.0                             & 6.9                            & 38.0                                      
                                        \\   %\hline
Ours                                                & 1.6                      & 2.0                      &   2.7                    & 2.9                       &                      & 1.9                       & 2.6                      & 1.9 (1.91)                     & 1.9                      & 4.3                                          & 0.9                             & 4.2                             & 6.3                                                                          %  \\
\\ \hline                                  
\end{tabular}
% }
\label{tab: Couple Laplacian loss V.S. Ours}
\end{table*}

\subsection{Comparison to standard Laplacian commutativity loss}

To demonstrate the superiority of filter learning in our frequency awareness preservation, we compare our approach with its most comparable alternative: the standard Laplacian commutativity loss. Table \ref{tab: Laplacian commutativity V.S. Ours} highlights the key differences between the two methods. The standard Laplacian commutativity loss is limited in that it cannot simultaneously supervise soft correspondence and functional maps, as its functional maps are computed solely using a functional map solver. To address this limitation, we introduce a modification to the standard Laplacian commutativity loss by replacing one of the functional maps $\mathbf{C}_\mathcal{NM}$ with $\mathbf{C}^{\Pi}_\mathcal{NM}$, ensuring that the functional map is obtained via soft correspondence.  This modification leads to the development of the Laplacian commutativity couple loss, defined as  $\left\|\mathbf{C}_{\mathcal{NM}}\Lambda_\mathcal{N} - \Lambda_\mathcal{M}\mathbf{C}^{\Pi}_{\mathcal{NM}}\right\|^{2}_\mathrm{F}$. Additionally, because the first eigenvalue is $0$, which does not meet our refinement condition during inference, we set the first eigenvalue to $1$ during both the training and inference phases.

The comparison results are shown in Table \ref{tab: Couple Laplacian loss V.S. Ours}, demonstrating that our method outperforms the Laplacian commutativity couple loss. The Laplacian commutativity couple loss can be interpreted as applying a high-pass filter($h(\lambda) = \lambda$) to the functional map. While this high-pass filter emphasizes high-frequency information and achieves strong performance in near-isometric matching scenarios, it performs suboptimally in non-isometric matching scenarios due to its disregard for low-frequency components. Additionally, the filter function in the Laplacian commutativity couple loss is single-channel and fixed, which limits its ability to capture multiscale frequency information or adapt to varying matching scenarios. In contrast, our method utilizes multi-channel, learnable filters. These filters not only enhance the generalization capabilities of our approach but also significantly improve matching accuracy, especially in cases involving large deformations or non-isometric correspondences.

\begin{figure}[h!t]
	\centering
	\includegraphics[width=0.98\linewidth]{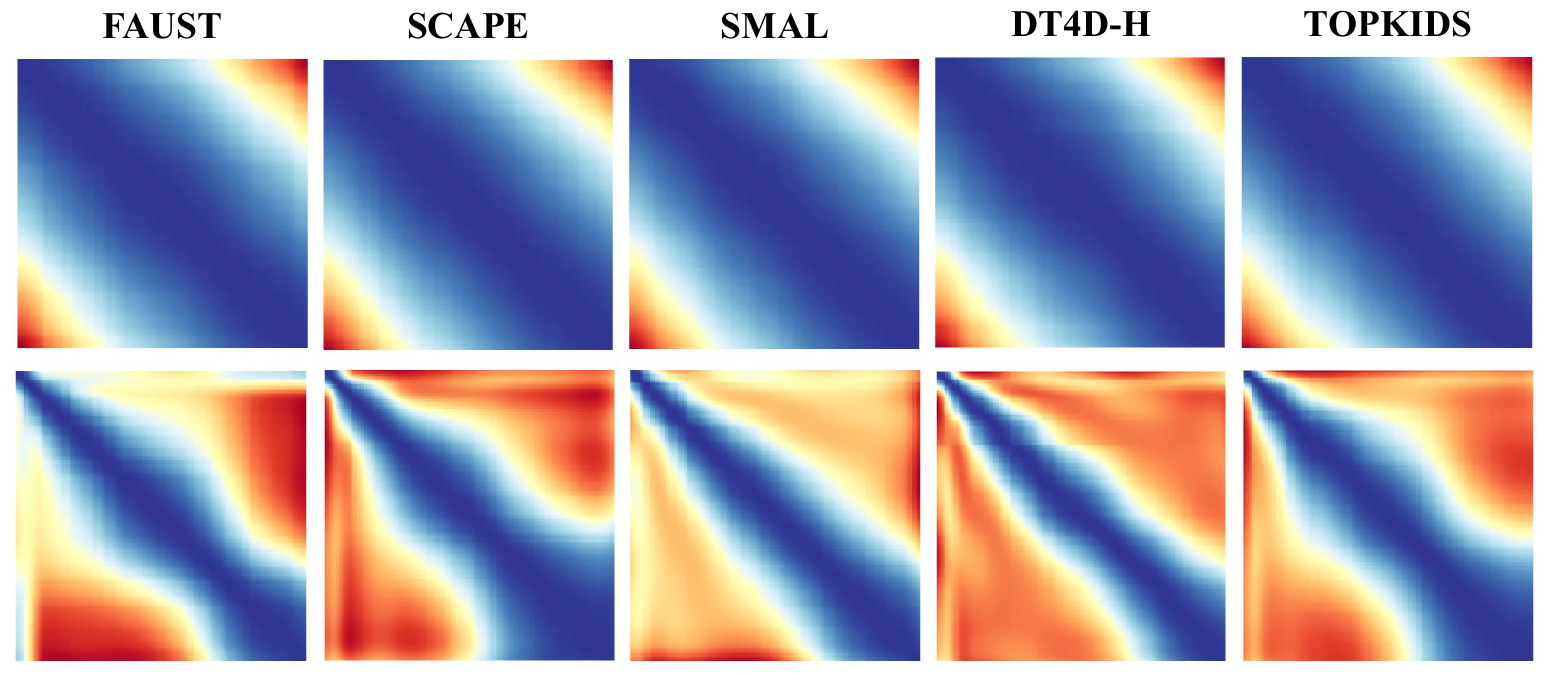}
	\caption{Filter analysis. We visualize the masks of the learned filters across different datasets (second row) and compare them with those derived from the standard Laplacian commutativity (first row). In these visualizations, hotter colors indicate higher penalties. Our observations reveal that our learned masks adaptively adjust the penalties on the functional map matrix, effectively preserving crucial matching information. In contrast, the masks from standard Laplacian commutativity  remain uniform across different datasets, failing to capture dataset-specific characteristics.}
	\label{fig:mask_filter}
\end{figure}

\subsection{Filter analysis}

To further analyze the role of our filters in constraining functional maps, we calculate the masks $ \mathbf{M}_{ij} = \sum_{s}\left( h_s(\lambda_i^{\mathcal{N}}) - h_s(\lambda_j^{\mathcal{M}}) \right)^{2}$ for each dataset separately and visualize them. Additionally, we compare our masks with those derived from the standard Laplacian commutativity \cite{Ovsjanikov2012} to highlight the superiority of our approach. Fig. \ref{fig:mask_filter} presents the qualitative results, where hotter regions represent larger penalties.

As shown in the results, the masks from Laplacian commutativity remain consistent across different datasets, such as FAUST and SMAL, indicating that standard Laplacian commutativity fails to adaptively preserve significant portions of the functional map matrix according to varying degrees of deformation. Since it encodes fixed eigenvalues
only, resulting in a lack of frequency awareness when restricting functional maps. In contrast, our masks adapt to the dataset-specific characteristics. For example, on near-isometric datasets like FAUST and SCAPE, our filters effectively penalize the non-diagonal regions of the functional map matrix, promoting diagonalization. This behavior aligns with the conclusion of functional maps \cite{Ovsjanikov2012}, which states that the functional map matrix is diagonal if two shapes are isometric. Penalizing non-diagonal regions maximally results in superior performance for near-isometric matching.

For datasets with non-isometric deformations, such as SMAL and DT4D-H, the learned filters slightly penalize the non-diagonal regions. This indicates that non-diagonal regions encode critical correspondence information for non-isometric matching, and retaining them as much as possible yields better results. For the TOPKIDS dataset, which contains topological noise, the shapes are near-isometric in the non-noisy regions but exhibit higher deformability compared to FAUST and SCAPE, and lower deformability compared to SMAL and DT4D-H. This is reflected in the visualized masks, where the non-diagonal penalties for TOPKIDS are lower than for FAUST and SCAPE, but higher than for SMAL and DT4D-H.

In conclusion, the experiments show that our learned filters adaptively filter out irrelevant information while retaining the essential correspondence information, depending on the matching scenario. This adaptive filtering enables superior performance across different datasets and deformation scenarios.

\begin{table}[h!t]
\centering
\caption{Ablation study on SMAL\cite{zuffi2017} dataset. }
\scalebox{1.0}{
\begin{tabular}{lcc}
\hline
               & \multicolumn{2}{c}{SMAL} \\ \hline
w.o $L_{freq}$           & \multicolumn{2}{r}{11.7}     \\ %\hline
w.o $L_{sm}$           & \multicolumn{2}{r}{4.6} \\
w.o learning filter layer & \multicolumn{2}{r}{8.4}     \\ %\hline
w.o our inference strategy        & \multicolumn{2}{r}{5.8}     \\ %\hline
Ours        & \multicolumn{2}{r}{\textbf{4.3}}     \\ \hline
\end{tabular}}
\label{tab: ablation study}
\end{table}

\subsection{Ablation study} 
The ablation studies we conduct on the challenging non-isometric dataset (e.g., SMAL \cite{zuffi2017}) to evaluate the importance of our three components, namely, the unsupervised loss in Section \ref{sec: loss}, the learning filter layer in Section \ref{sec: MSFOP}, and the inference strategy in Section \ref{sec: refinement}.

The results are summarized in Table \ref{tab: ablation study}. The first and second rows show the network trained without $L_{freq}$ and $L_{sm}$, respectively. The third row indicates the spectral filter functions $H(\Lambda)$ obtained without updating their weight parameters. The fourth row represents the pointwise map obtained at inference without using our inference strategy, i.e., acquiring the pointwise map directly via Eq. \eqref{equ: nnsearch in test}. By comparing the first row and the last row, we can conclude that $L_{freq}$ indeed improves matching performance by a large margin. By comparing the second row and the last row, we observe that $L_{sm}$ can improve the quantitative results. Comparing the third row and the last row shows that the learning filter layer plays an important role in accurate matching. Comparing the fourth row and the last row confirms that using our inference strategy yields better matching performance.

\begin{figure}[h!t]
	\centering
	\includegraphics[width=0.7\linewidth]{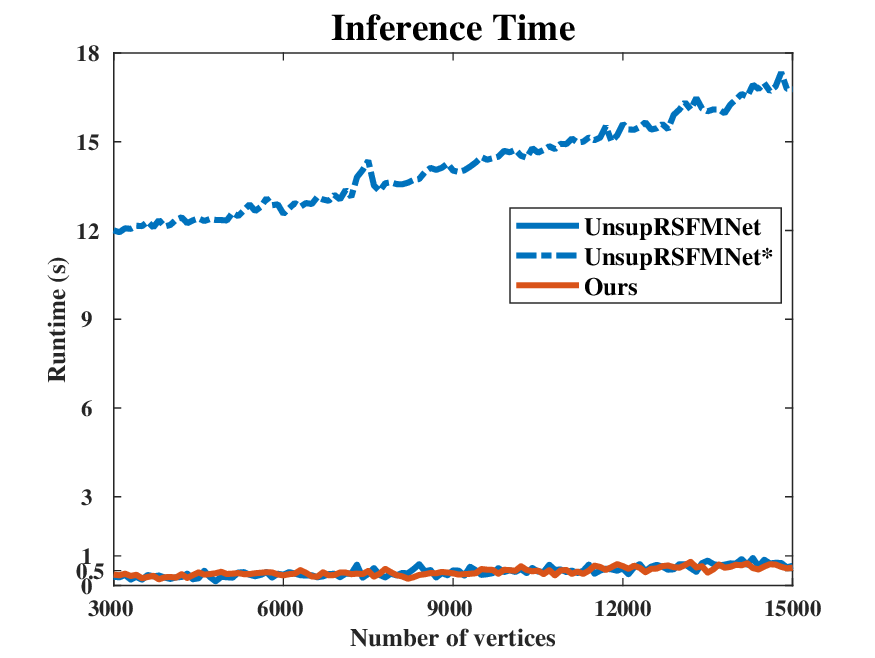}
	\caption{Runtime comparison with different numbers of vertices. UnsupRSFMNet$^*$ refers to the use of test-time adaptation for refinement, which requires significantly more computational time. Our method achieves a runtime comparable to UnsupRSFMNet while delivering better results. Compared to UnsupRSFMNet*, our method still achieves comparable results but is significantly faster to execute.}
	\label{fig: Runtime}
\end{figure}

\subsection{Runtime comparison}

We evaluate the runtime of our method and compare it to the state-of-the-art unsupervised method UnsupRSFMNet\cite{Cao2023}. Fig. \ref{fig: Runtime} illustrates the runtime for shapes with a different number of vertices ranging
from 3K to 15K when using the same setting of 200
eigenfunctions. We observe that the computational speed of our method is comparable to that of UnsupRSFMNet, and tens of times faster than UnsupRSFMNet$^*$. Despite having about the same computational efficiency, our matching results are much better than UnsupRSFMNet. To further enhance the matching performance, UnsupRSFMNet$^*$ uses a fine-tuning technique called test-time adaptation to update the parameters of the network during testing for each test pair individually via 15 backpropagation iterations, which is seriously time-consuming. A typical example is that the inference time of UnsupRSFMNet$^*$ on SMAL\cite{zuffi2017} is approximately equal to $1/3$ of the training time. Conversely, our method is more efficient than it and does not require any parameter updates to achieve superior performance. All the statistics were collected on a server with Intel(R) Xeon(R) Platinum 8358 CPU @ 2.60GHz, and a single NVIDIA A100-958 SXM4-80GB GPU.

\subsection{Parameter analysis} 
Learning spectral filters plays a critical role in our architecture, and selecting the appropriate number of filters is essential for improving matching results. To analyze this, we evaluate two parameters: the Jacobi polynomial order $L$ and the number of channels in the spectral filters $S$. The results of the parameter analysis are summarized in Table \ref{tab: parameter analysis}.

The results indicate that our method is relatively insensitive to both the Jacobi polynomial order and the number of filters, as varying $S$ for a fixed $L$ or varying $L$ for a fixed $S$ has minimal impact on the results. While the best performance is achieved with $L=16$ and $S=8$, higher-order Jacobi polynomials offer better expressiveness but also increase the risk of overfitting, which negatively affects the generalization performance of the network. Therefore, we select the suboptimal parameters  $L=8$ and $S=6$ as the final configuration to strike a balance between performance and generalization.

\begin{table}[h!t]
\centering
\caption{Parameter Analysis on SMAL \cite{zuffi2017}: We evaluate the correspondence accuracy on the SMAL dataset to determine the optimal parameter values for our method. Specifically, $L$ refers to the Jacobi polynomial orders, and $S$ denotes the number of channels in the spectral filters.}
\scalebox{0.96}{
\begin{tabular}{cccccccccccc}
\hline
$L$  & \multicolumn{3}{c}{32} &  & \multicolumn{3}{c}{16} &  & \multicolumn{3}{c}{8} \\ \cline{2-4} \cline{6-8} \cline{10-12} 
$S$ & 16     & 10     & 8    &  & 16     & 10     & 8    &  & 8     & 6     & 4     \\ \hline
 Geo error      & 5.0       & 4.4       & 5.1     &  &  4.9      & 4.8       & \textbf{4.2}     &  & 4.8      & \underline{4.3}     & 4.8       \\ \hline
\end{tabular}}
\label{tab: parameter analysis}
\end{table}

\begin{figure}[h!t]
    \centering
    \includegraphics[width=0.95\linewidth]{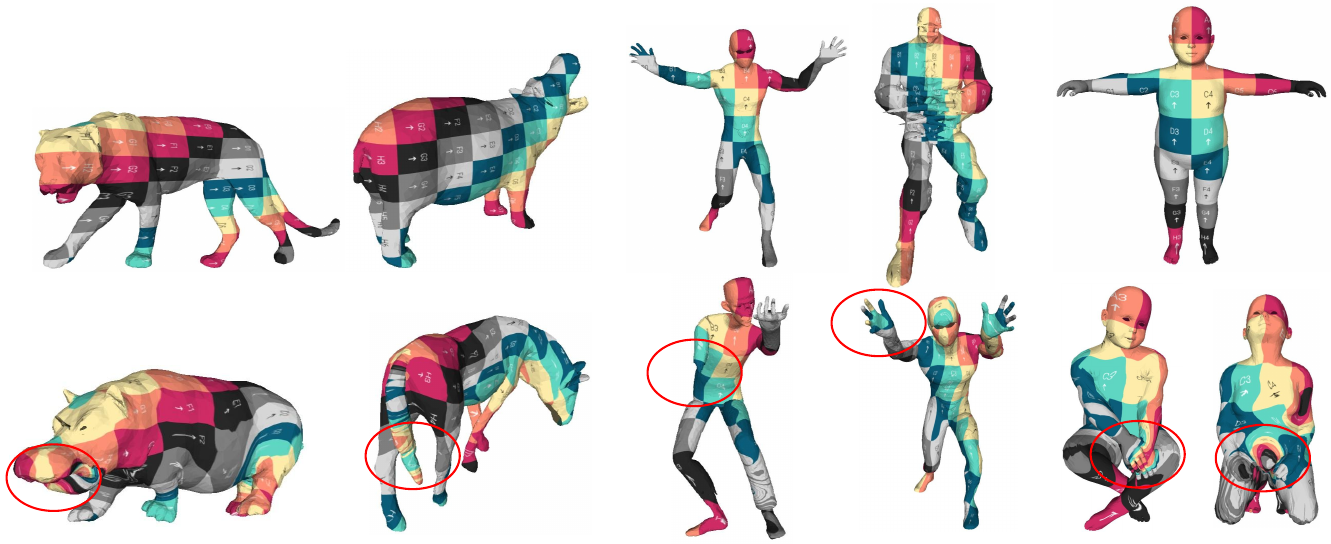}
    \caption{Examples of failure cases of our method. The first row displays the source shapes, while the second row shows the target shapes. The first and second columns illustrate cases with large distortions between the source and target models, such as significant discrepancies between the mouth and the tail (SMAL). In the third column, the source model’s right hand is intact, whereas the target model’s right hand is mutilated. In the fourth column, the source model has crossed hands, while the target model’s hands are separated. These structural differences result in matching failures (DT4D-H). Lastly, the fifth and sixth columns show examples of severe topological noise in the target model, which may also lead to local mismatches (TOPKIDS).}
    \label{fig:bad example}
\end{figure}
\section{Discussion and Limitations}
Our constraint is induced from the isometric assumption, which explicitly models the functional maps, pointwise maps, and a set of filter functions. By using the proposed constraint as the loss function, as well as using extracted features compute both functional maps and pointwise maps, in a differentiable way during training. Our approach thus ensures their simultaneous optimization and promotes their frequency awareness, resulting in not only do the learned features thus extract accurate and robust shape features, but the data-driven filter functions encode significant and informative frequency information. Moreover, the optimized filter functions and extracted features are adequately adopted by our refinement strategy to compute the final correspondence during inference, leading to more accurate and robust matching results impulsively. Therefore, our approach can gracefully handle a range of matching scenarios, even though non-isometric matching.

Although our method achieves state-of-the-art performance in many challenging matching scenarios, its accuracy decreases under certain extreme conditions. Figure \ref{fig:bad example} presents failure cases that highlight the challenges of matching in harsh environments, such as large non-isometric deformations and severe topological noise. Additionally, like other deep functional map methods, our approach requires computing a dense soft correspondence matrix, which demands excessive memory for shapes with a large number of vertices. To mitigate this issue, spectral-preserving simplification techniques \cite{yazgan2023partition} can be employed to accelerate computation by significantly reducing the complexity of the mesh representation while maintaining accuracy.

Furthermore, our method relies on the Polynomial Coefficient Decomposition technique (see Section~\ref{app: Jacobi}) to enhance generalization performance. 
% However, representing filter functions using orthonormal polynomials may not be the optimal approach, 
As orthonormal polynomials tend to introduce higher-order frequency information and may lead to overfitting. A promising direction for future work could involve representing filter functions with smooth and geometrically meaningful basis functions, such as heat kernels, wavelet kernels, or other alternatives. This approach can inspire further exploration into designing better spectral filters for improved performance, much like how spectral filter design is a
critical aspect of spectral convolution in graph neural networks. We anticipate this will be a highly interesting and valuable research direction for the future.
\section{Conclusion}
We propose a novel and general constraint to enhance frequency awareness for functional map estimation. Building on this constraint, we develop a learning-based architecture for deformable shape matching by leveraging a set of learnable spectral filters. Our constraint serves a dual purpose: it can be employed as a loss function to couple functional maps, pointwise maps, and filter functions during the training phase, and it can also act as an effective refinement strategy to improve the final pointwise map for more robust and accurate correspondences during inference. Extensive experiments across diverse datasets demonstrate that our approach outperforms existing state-of-the-art methods, particularly in challenging scenarios such as non-isometric and inconsistent topology datasets. The source code to replicate our results will be available upon publication.
\section{Acknowledgments}
We wish to thank the anonymous reviewers for their valuable comments and helpful suggestions. This work was supported by the Natural Science Foundation of China (Nos.62172447, 62271452, 62302530), National Key Research and Development Program of China (No.2023YFB4502400), Hunan Provincial Natural Science Foundation of China (No.2023JJ40769), the Key Project of Education Department of Hunan Province (No.23A0641), and the Changsha Natural Science Foundation (No.kq2402092). Additionally, we are grateful to the High Performance Computing Center of Central South University for partial support of this work.

\bibliography{DeepFAFP}
\bibliographystyle{IEEEtran}

% \newpage
% \input{sec/10_bio}

\vfill

\clearpage
% \appendix
\setcounter{page}{15}
\twocolumn[
    \centering
    \Large
    \textbf{Deep Frequency Awareness Functional Maps for Robust Shape Matching} \\
    \vspace{0.5em}Supplementary Document \\
    \vspace{1.0em}
] %< twocolumn
\setcounter{section}{0}
\setcounter{table}{0}
\setcounter{figure}{0}
\setcounter{equation}{0}
\renewcommand{\thetable}{\thesection\arabic{table}}
\renewcommand{\thefigure}{\thesection\arabic{figure}}
\renewcommand{\theequation}{\thesection\arabic{equation}}

\section{Proofs of Remark \ref{pro: remark_2}}\label{app: Pro}

For convenience, we restate the remark as follows:
\newtheorem{thm_filter}{Remark}[section]
\begin{thm_filter}
	~If the set of filter functions $\{h_s(\lambda)\}_{s=1}^S$ satisfy the Consistency Condition that 
	$G(\lambda) = \sum_{s} h_s^{2}\left(\lambda\right) \neq 0, \forall \lambda,$
then  the functional map $ \mathbf{C_\mathcal{NM}} $ in Eq.( \ref{equ: compute C by fixing Pi})
can be obtained via 
\begin{equation*}
% \label{eq: MCFP compute C}
\begin{aligned}
	\mathbf{C}_\mathcal{NM} =  \left(\sum_{s} h_s\left(\Lambda_{\mathcal{M}}\right) \mathbf{C}^{\Pi}_\mathcal{NM} h_s\left( \Lambda_{\mathcal{N}}\right)\right)  {\left(G^{-1}(\Lambda_{\mathcal{N}})\right)}, 
 \end{aligned}
\end{equation*}
where $G(\Lambda_{\mathcal{N}}) =\sum_{s} h_s^{2}\left(\Lambda_{\mathcal{N}}\right) $.

\begin{proof}
Since the spectral filter operator is invariant to isometric deformation, if the pointwise map $\Pi_\mathcal{MN} $ also is an isometric map between shapes, we can obtain an analytic solution for $ \mathbf{C}_\mathcal{NM} $ via least square sense, the detailed derivation is given below: 
$$\begin{aligned}
	\mathbf{C}_\mathcal{NM}  
	 H\left(\Lambda_{\mathcal{N}}\right) & = H\left(\Lambda_{\mathcal{M}}\right)*\mathbf{C}^{\Pi}_\mathcal{NM}, \\
  	\mathbf{C}_\mathcal{NM}  H\left(\Lambda_{\mathcal{N}}\right) H\left(\Lambda_{\mathcal{N}}\right)^{\mathrm{T}}   &= \left(H\left(\Lambda_{\mathcal{M}}\right)*\mathbf{C}^{\Pi}_\mathcal{NM}\right) H\left( \Lambda_{\mathcal{N}}\right)^{\mathrm{T}},\\
	\mathbf{C}_\mathcal{NM} \sum_{s}  h_s^{2}\left(\Lambda_{\mathcal{N}}\right) & =\sum_{s}   h_s\left(\Lambda_{\mathcal{M}}\right) \mathbf{C}^{\Pi}_\mathcal{NM} h_s\left( \Lambda_{\mathcal{N}}\right).
\end{aligned}$$
As $ G(\Lambda_{\mathcal{N}}) =\sum_{s} h_s^{2}\left(\Lambda_{\mathcal{N}}\right) \in \mathbb R ^{k \times k} $ is a diagonal matrix, if it satisfies the condition 
\begin{equation}\label{MCFP condition}
    G(\lambda) = \sum_{s} h_s^{2}\left(\lambda\right) \neq 0, \forall \lambda,
\end{equation}	
the matrix $ G(\Lambda_{\mathcal{N}})$ is invertible and computing its inverse is easy, and finally we have Eq.(\ref{MCFP compute C}).
\end{proof}

\end{thm_filter}

\section{Jacobi polynomials for Filter Functions}\label{app: Jacobi}
We will provide the details of Jacobi polynomials stated in Section \ref{sec: MSFOP} in the following. Among orthogonal polynomials, the Jacobi polynomials has a very general form, where some orthogonal bases are considered as its special cases, such as Chebyshev basis, Legendre basis, etc. The form of Jacobi polynomials $J_{l}^{a, b}$ can be expressed as 
\begin{equation*}
J_{0}^{a, b}({\lambda})=1, \\
J_{1}^{a, b}({\lambda})=\frac{a-b}{2}+\frac{a+b+2}{2} {\lambda},
\end{equation*}
for $l>=2$. 
\begin{equation*}
J_{l}^{a, b}({\lambda})=\left(\mu_{l} {\lambda}+\mu_{l}^{\prime}\right) J_{l-1}^{a, b}({\lambda})-\mu_{l}^{\prime \prime} J_{l-2}^{a, b}({\lambda}),
\end{equation*}
where
\begin{equation*}
\begin{aligned}
\mu_{l} & =\frac{(2l+a+b)(2l+a+b-1)}{2l(l+a+b)}, \\
\mu_{l}^{\prime} & =\frac{(2l+a+b-1)\left(a^{2}-b^{2}\right)}{2l(l+a+b)(2l+a+b-2)}, \\
\mu_{l}^{\prime \prime} & =\frac{(l+a-1)(l+b-1)(2 l+a+b)}{l(l+a+b)(2l+a+b-2)}. 
\end{aligned}
\end{equation*}
$J_{l}^{a, b}, l=0, 1, 2, ...$ are orthogonal w.r.t. the weight function $\left(1- \lambda \right)^{a} \left(1 + \lambda \right)^{b}$ on $[-1,1]$. 
$J_{l}^{a, b}$ is Chebyshev basis when $a =b= \pm \frac{1}{2}$, and  is Lengende basis when $a =b=0$. Furthermore, 
according to~\cite{tao2023longnn}, $a$ and $b$ can also be deﬁned as learnable parameters in  our Learning Filter layer, which could  generate more proper $a$ and $b$ for Jacobi polynomials.

Since the Jacobi polynomials are defined in the interval $ [-1, 1] $, the truncated eigenvalues $\Lambda$ of the LBOs should be scaled to this interval accordingly. Thus, we rewrite the orthonormal Jacobi polynomials~\cite{tao2023longnn}  as following:  
\begin{equation*}\label{eq: otrhnomalize}
\begin{aligned}
& \left\|{J}^{a,b}_l(\widetilde{\lambda})\right\|=\sqrt{\frac{2^{a+b+1} \Gamma(l+a+1) \Gamma(l+b+1)}{(2l+a+b+1)\Gamma(l+a+b+1)l!}} \\
& {J}_l(\widetilde{\Lambda})=\frac{J_l^{a,b}(\widetilde{\Lambda})}{\left\|J_l^{a,b}(\widetilde{\lambda})\right\|}.
\end{aligned}
\end{equation*}
where $ \widetilde{\Lambda} = \frac{2\Lambda}{\lambda_{k}}-\mathrm{I}_{k} $, $ \mathrm{I}_{k} \in \mathbb{R}^{k \times k} $ is a identity matrix, and $\Gamma$ is Gamma function. Then we have the learnable orthonormal Jacobi polynomials for yielding our filters, e.g. 
\begin{equation*}
    h_s(\widetilde{\Lambda}) = \sum_{l} w_{s l} {J}^{a,b}_l(\widetilde{\Lambda}),
\end{equation*}
where$\{w_{s l}\}$ are trainable coefficients. In addition, in order to reduce the overfitting of our networks, 
we use the technique called Polynomial Coefficient Decomposition (PCD) \cite{wang2022powerful} to decompose the polynomial coefficients. 
$w_{sl} $ could be decomposed to  $w_{sl}=\alpha_{sl} \prod \limits_{i=0}^n \beta_{i}$, then we can modify the recursion formula to implement PCD.
\begin{equation*}
    {J}_{l}^{a, b}(\widetilde{\Lambda})=\beta_{l} (\mu_{l} \widetilde{\Lambda}+\mu_{l}^{\prime}) {J}_{l-1}^{a, b}(\widetilde{\Lambda})- \beta_{l}\beta_{l-1} \mu_{l}^{\prime \prime} {J}_{l-2}^{a, b}(\widetilde{\Lambda}), 
\end{equation*}
where $\beta_{i} = \beta^{'} tanh \gamma_{i} $, which enforces $\beta_{i}\in \left[-\beta^{'}, \beta^{'} \right]$.

% \vspace{11pt}

\end{document}